\documentclass[11pt]{article}

\usepackage[numbers]{natbib}

\usepackage{booktabs}       
\usepackage{amsfonts}       
\usepackage{amsmath}
\usepackage{graphicx}
\usepackage{amsthm}
\usepackage{amssymb}
\usepackage{multirow}
\usepackage{makecell}
\usepackage[vlined,linesnumbered,ruled,resetcount]{algorithm2e}
\usepackage[colorlinks,linkcolor=magenta,filecolor=blue,citecolor=blue,urlcolor=blue]{hyperref}%
\usepackage[top=1in, left=1in, right=1in, bottom=1in]{geometry}
\usepackage{subfigure}
\usepackage{amssymb}
\usepackage{pifont}
\usepackage{mathtools}
\usepackage{stmaryrd}
\usepackage[T1]{fontenc}

\usepackage{authblk}
\usepackage{enumitem}
\usepackage{algorithmic}

\newtheorem{theorem}{Theorem}[section]
\newtheorem{lemma}[theorem]{Lemma}

\newtheorem{corollary}[theorem]{Corollary}

\newtheorem{assumption}[theorem]{Assumption}

\theoremstyle{definition}
\newtheorem{definition}[theorem]{Definition}

\newtheorem{remark}[theorem]{Remark}

\newcommand{\wh}{\widehat}
\newcommand{\wt}{\widetilde}
\newcommand{\ov}{\overline}

\renewcommand{\epsilon}{\varepsilon}
\renewcommand{\phi}{\varphi}

\renewcommand{\tilde}{\wt}
\renewcommand{\hat}{\wh}
\renewcommand{\bar}{\ov}

\makeatletter
\newcommand*{\RN}[1]{\expandafter\@slowromancap\romannumeral #1@}
\makeatother

\makeatletter
\newcommand{\printfnsymbol}[1]{%
  \textsuperscript{\@fnsymbol{#1}}%
}
\makeatother

\usepackage[utf8]{inputenc} % allow utf-8 input
\usepackage[T1]{fontenc}    % use 8-bit T1 fonts
\usepackage{hyperref}       % hyperlinks
\usepackage{url}            % simple URL typesetting
\usepackage{booktabs}       % professional-quality tables
\usepackage{amsfonts}       % blackboard math symbols
\usepackage{nicefrac}       % compact symbols for 1/2, etc.
\usepackage{microtype}      % microtypography
\usepackage{xcolor}         % colors
\newcommand{\methodname}{Sketch\&Walk}
\usepackage{enumitem}
\usepackage{graphicx}
\usepackage{amsmath}
\usepackage{amssymb}
\usepackage{mathtools}
\usepackage{amsthm}
\usepackage{bm}
\usepackage{float}
\usepackage[export]{adjustbox}
\usepackage{soul}
\usepackage{algorithmic}
\usepackage{booktabs}
\usepackage{pgf-pie}
\usepackage{caption}
\usepackage{multirow}
\usepackage{makecell}
\usepackage{wrapfig}
\usepackage[utf8]{inputenc}
\usepackage[T1]{fontenc}
\usepackage{subcaption}
\usepackage{xcolor}
\usepackage{wrapfig}
\usepackage{thm-restate}

\usepackage{adjustbox}
\usepackage{array}

\newcolumntype{R}[2]{%
  >{\adjustbox{angle=#1,lap=\width-(#2)}\bgroup}%
  l%
  <{\egroup}%
}

\newcommand{\rot}[1]{\multicolumn{1}{R{25}{0em}}{\shortstack{#1}}}

\title{Scout Before You Attend: Sketch-and-Walk Sparse Attention \\ for Efficient LLM Inference}

\author[1]{Hoang Anh Duy Le}
\author[1]{Sahil Joshi}
\author[1]{Zeyu Yang}
\author[2]{Zhaozhuo Xu}
\author[1]{Anshumali Shrivastava}
\affil[1]{Department of Computer Science, Rice University}
\affil[2]{Department of Computer Science, Stevens Institute of Technology}
\affil[ ]{{\texttt{\{Escanord.Le, Sahil.Joshi, Zeyu.Yang, Anshumali.Shrivastava\}@rice.edu}, \texttt{zxu79@stevens.edu}}}
\date{}

\begin{document}

\maketitle
\begin{abstract}
    Self-attention dominates the computational and memory cost of long-context LLM inference across both prefill and decode phases. 
To address this challenge, we introduce \textbf{Sketch\&Walk} Attention, a training-free sparse attention method that determines sparsity with lightweight sketches and deterministic walk.
Sketch\&Walk applies Hadamard sketching to get inexpensive approximations of attention scores, then aggregates these estimates across layers via a walk mechanism that captures attention influence beyond direct interactions between tokens.
The accumulated walk scores are used to select top-$k$ attention blocks, enabling dynamic sparsity with a single training-free algorithm that applies uniformly to both the prefill and decode phases, together with custom sparse attention kernels.
Across a wide range of models and tasks, Sketch\&Walk maintains near-lossless accuracy at 20\% attention density and can slightly outperform dense attention in some settings, while achieving up to $6\times$ inference speedup.
\end{abstract}

\section{Introduction}

Large Language Models (LLMs) have demonstrated remarkable capabilities, yet their deployment at scale remains hindered by the quadratic cost of self-attention~\cite{vaswani2017attention}, which becomes prohibitive as context lengths grow to hundreds of thousands of tokens \citep{pope2023efficiently_scaling_tf, fu2024challenges_lctx_tf}. Numerous efficient inference algorithms have been proposed to address this bottleneck~\cite{yuan2024longctx_bench}, including sparse attention methods that compute only a subset of query-key interactions~\citep{jiang2024minference,tang2024quest,lai2025flexprefill,yan2025adamas}. Despite their effectiveness, most existing sparse attention methods rely on a common design choice: they select key blocks based on \emph{one-hop} attention scores computed independently at each layer. This one-hop selection inherently misses \emph{multi-hop} token connections (where a query block $i$ depends on block $k$ through intermediate blocks $j$) that emerge only through repeated attention composition in deep transformers~\citep{abnar2020quantifying,cai2024pyramidkv}.

\noindent \textbf{One-hop-per-layer selection misses multi-hop token dependencies.}
In transformers, information propagates through layers via repeated attention composition. As a result, the dependence of a token at position $i$ on another token $k$ may arise through a sequence of intermediate tokens, rather than through a single direct interaction.
However, attention scores are inner products, a non-metric similarity that violates transitivity: if block $i$ attends strongly to $j$, and $j$ to $k$, this does \emph{not} imply $i$ directly attends strongly to $k$. 
Yet, existing sparse attention methods compute one-hop scores at each layer and rely on cross-layer stacking to implicitly capture multi-hop dependencies. However, this might \emph{fail}: if an intermediate token $j$ is not selected at layer $\ell$, any multi-hop dependency of the form $i \!\to\! \ldots \!\to\! j \!\to\! k$ is broken, and subsequent layers cannot recover it. Moreover, this limitation \emph{cannot be resolved} by performing multi-step aggregation within a single layer, since attention at one layer reflects only one-hop similarity, whereas multi-hop paths require information to actually propagate through intermediate tokens, which happens \emph{across} layers.

\begin{figure}[t]
\vspace{-1em}
    \centering
    \includegraphics[width=0.9\linewidth]{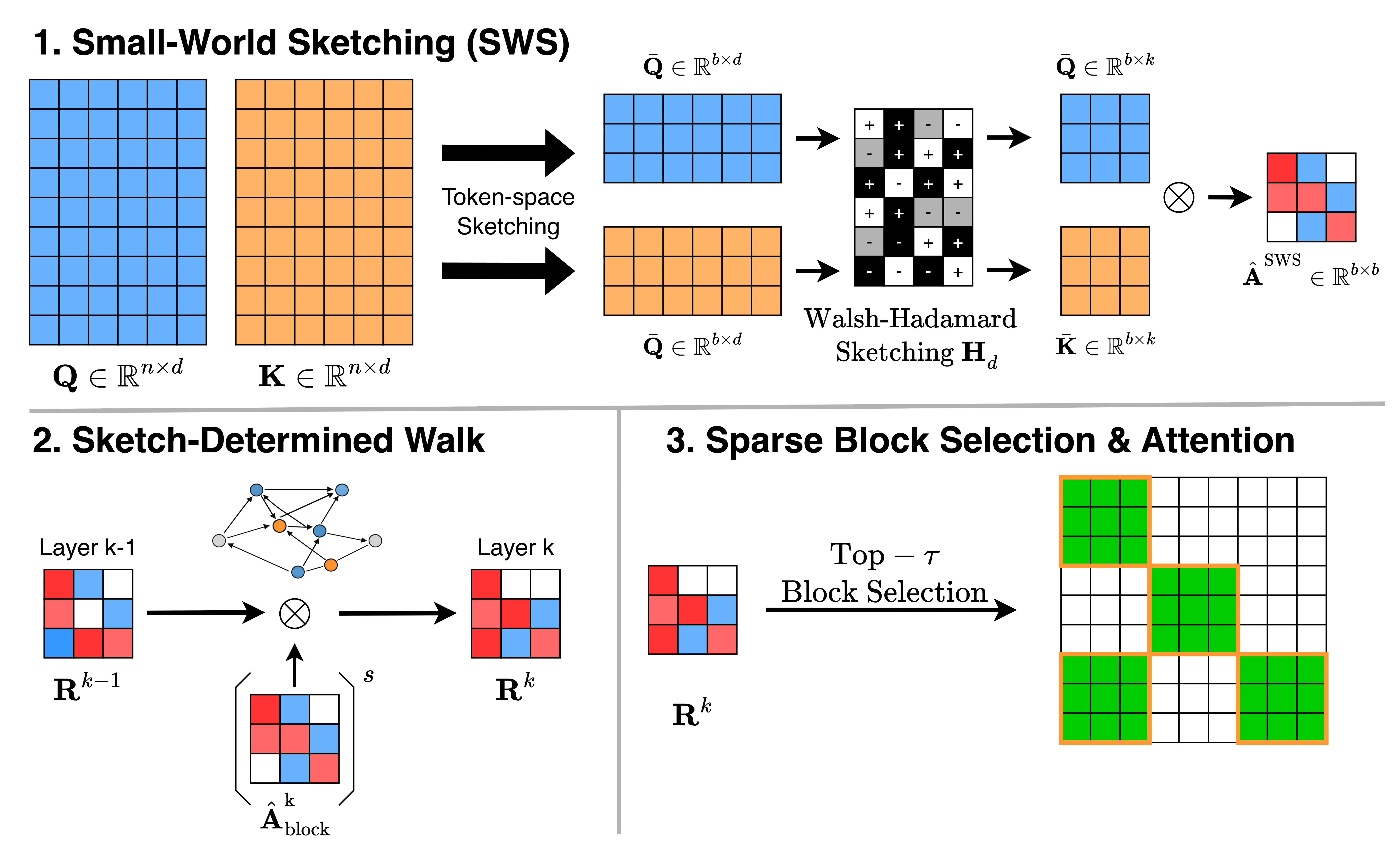}
    \caption{
    Overview of \methodname{}. 
(1) Queries and keys are sketched with Small-World Sketching to obtain lightweight block-level attention estimates. 
(2) These estimates are accumulated across layers with Sketch-Determined Walk to approximate cross-layer attention influence. 
(3) The resulting walk scores are used to select top-$\tau$ blocks for sparse attention.
}
    \label{fig:sketch_walk}
    \vspace{-2em}
\end{figure}

\noindent \textbf{Intra-layer walks on full attention are prohibitively expensive, but sketching makes them tractable.}
The natural solution is to perform a walk on the attention matrix \emph{within} each layer to capture multi-hop dependencies: computing powers $A^s$ reveals which tokens are connected through multi-hop paths at that layer. However, this is prohibitively expensive, computing the full $n \times n$ dense attention matrix negates any efficiency gains from sparse attention. Our key insight is that we can perform a \emph{blockwise} walk instead: by aggregating tokens into blocks via \emph{token-space sketching} and compressing features via \emph{feature-space sketching}, we reduce the intra-layer walk to $b \times b$ block matrices where $b \ll n$. This makes multi-hop aggregation tractable at inference time without retraining.

\noindent \textbf{We propose \methodname{}: sketch to estimate, walk to aggregate.}
We introduce \methodname{} Attention, which captures multi-hop dependencies through a two-stage approach that operates entirely at inference time. First, \textbf{Small-World Sketching (SWS)} efficiently estimates block-level attention via token-space sketching and feature-space sketching, reducing multiplication complexity from $O(B^2 d)$ to $O(k)$ per block pair. Second, \textbf{Sketch-Determined Walk} uses these sketched attention estimates to perform a walk \emph{at inference time per layer}: $R^k = R^{k-1} (\hat{\mathbf{A}}_{\text{block}}^k)^{s}$. \textbf{This per-layer walk during inference is our key novelty}, enabling multi-hop aggregation \emph{within} each layer's forward pass, rather than relying on depth-wise composition of attention layers. The \emph{single-layer} walk state accumulates multi-hop paths connecting query and key blocks across the transformer depth. By selecting sparse attention targets from the walk state rather than one-hop scores, \methodname{} retains blocks that are important and can only be identified with the composition of attention layers. We establish provable guarantees showing that \methodname{} recovers the correct sparse attention pattern with high probability.

\noindent \textbf{\methodname{} achieves near-lossless accuracy with substantial inference speedups.}
Across multiple model scales and a broad range of tasks, \methodname{} maintains near-lossless accuracy at 80\% sparsity and can slightly outperform dense attention in some settings.
The same inference-time mechanism applies uniformly to both the prefill and decode phases.
Combined with custom Triton kernels, \methodname{} delivers consistent speedups that grow with context length, achieving up to $6\times$ improvement in end-to-end inference throughput. Our contributions include:
\begin{itemize}[nosep,leftmargin=*]
    \item We identify a fundamental limitation of per-layer, one-hop sparse attention selection: it might fail to select tokens involved in multi-hop dependencies (induced by attention composition) and required in later layers.
    \item We propose \methodname{}, a training-free, low-overhead sparse attention method that applies to both prefill and decode phases of the inference process, capturing multi-hop token dependencies through Small-World Sketching and Sketch-Determined Walk.
    \item We establish theoretical approximation guarantees of \methodname{} and empirically demonstrate \methodname{} can achieve up to $6\times$ inference speedup while preserving near-lossless accuracy across long-context benchmarks.
\end{itemize}

%%%%%%%%%%%%%%%%%%%%%%% METHOD

\begin{figure*}
    \centering
    \includegraphics[width=0.95\linewidth]{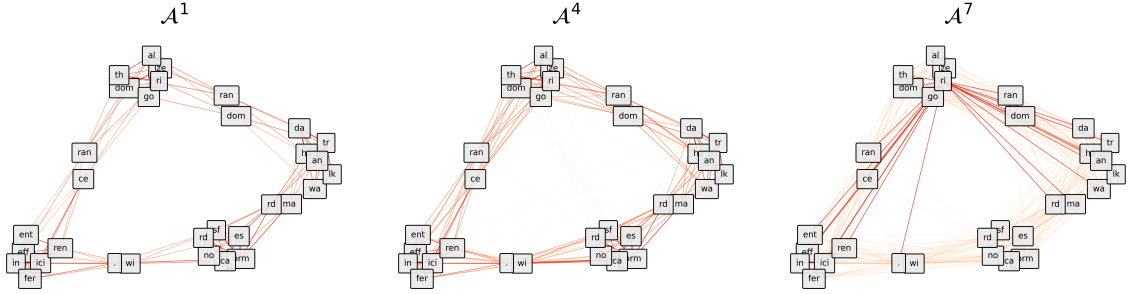}
    \caption{
Visualization of attention matrices from a layer of the Llama-3.1-8B-Instruct.
Each node corresponds to a token, and edge intensity reflects the magnitude of the attention score.
$A^1$ (left) shows direct attention scores.
Higher-order compositions of attention are shown by $A^4$ (middle) and $A^7$ (right).
While $A^1$ captures only direct interactions, higher powers approximate the influence induced by repeated attention composition, reflecting attention that becomes strong in deeper layers.
}
    \label{fig:random_walk}
    \vspace{-1em}
\end{figure*}

\section{Sketch and Walk}

We introduce \textbf{Sketch\&Walk Attention}, based on \textbf{Small-World Sketching (SWS)} followed by a \textbf{Sketch-Determined Walk}. 
SWS applies \emph{token-space sketching} via block-level aggregation and \emph{feature-space sketching} via Hadamard projections to estimate block-level attention scores, which are then used to guide the Sketch-Determined Walk operation over attention blocks.
The name reflects the key insight that small-world sketching induces a small-world network structure, where any token can reach another through a small number of high-importance block transitions, mirroring the classical small-world phenomenon in social networks. See Figure \ref{fig:sketch_walk} for a complete overview of \methodname{} algorithm.

\subsection{Preliminaries and Notation}

Let $\mathbf{Q}, \mathbf{K} \in \mathbb{R}^{n \times d}$ denote the query and key matrices with $n$ tokens and head dimension $d$. We partition these into $b = \lceil n/B \rceil$ blocks of size $B$. For block $i$, let $\mathbf{Q}^{(i)}, \mathbf{K}^{(i)} \in \mathbb{R}^{B \times d}$ denote the corresponding sub-matrices. 

\begin{definition}[Block Attention Score]
\label{def:block_attention}
The true block attention score between query block $i$ and key block $j$ is defined as:
$$
A_{ij}^{\text{true}} = \frac{1}{\sqrt{d}} \frac{1}{B^2} \sum_{s=1}^{B} \sum_{t=1}^{B} \mathbf{Q}^{(i)}[s,:] \cdot \mathbf{K}^{(j)}[t,:]^\top.
$$
\end{definition}

\subsection{Small-World Sketching}

\textbf{Small-World Sketching (SWS).}
We formalize \emph{Small-World Sketching} as a two-stage sketching operator that constructs a low-dimensional representation of token interactions for efficient block-level attention estimation.

\begin{itemize}[nosep,leftmargin=*]
    \item \textbf{Token-space sketching (block aggregation).}
    Given a sequence of token representations partitioned into blocks of size $B$, the $i$-th block representation is defined as
    \[
        \bar{\mathbf{q}}_i \;\triangleq\; \frac{1}{B}\sum_{t=1}^{B} \mathbf{Q}^{(i)}[t,:],
        \quad
        \bar{\mathbf{k}}_i \;\triangleq\; \frac{1}{B}\sum_{t=1}^{B} \mathbf{K}^{(i)}[t,:] \in \mathbb{R}^{d},
    \]
    which reduces the original sequence of $n$ tokens to $b=\lceil n/B\rceil$ block-level vectors.
    
    \item \textbf{Feature-space sketching (Hadamard projection).}
    Each block representation mapped to a lower-dimensional feature space via a randomized Hadamard projection,
    \[
        \tilde{\mathbf{q}}_i \;\triangleq\; \bar{\mathbf{q}}_i \mathbf{H}_d \in \mathbb{R}^{k},
        \qquad
        \tilde{\mathbf{k}}_i \;\triangleq\; \bar{\mathbf{k}}_i \mathbf{H}_d \in \mathbb{R}^{k},
    \]
    where $\mathbf{H}_d \in \mathbb{R}^{d \times k}$ is a Hadamard matrix and $k \ll d$ denotes the sketch dimension. The orthogonality of the Hadamard transform ensures that similarity estimates in the transformed space remain consistent with those in the original feature space.
\end{itemize}

Together, token-space and feature-space sketching produce a compact block-level sketch that supports efficient estimation of inter-block attention scores. In practice, we implement custom CUDA kernels for block aggregation and fast Hadamard transforms to reduce the runtime overhead of Small-World Sketching. Under Small-World Sketching, the block-level attention matrix is estimated as:
\begin{align*}
    \hat{\mathbf{A}}^{\text{SWS}}
    \;\triangleq\;
    \frac{\tilde{\mathbf{Q}}\,\tilde{\mathbf{K}}^\top}{\sqrt{k}}
    \;=\;
    \frac{(\bar{\mathbf{Q}}\mathbf{H}_d)\,(\bar{\mathbf{K}}\mathbf{H}_d)^\top}{\sqrt{k}},
\end{align*}
This reduces the complexity of computing block scores from $O(B^2 d)$ to $O(k)$ per block pair.

\subsection{Sketch-Determined Walk}

After estimating block-level attention scores via \emph{Small-World Sketching (SWS)}, \methodname{} applies a \emph{Sketch-Determined Walk} to aggregate these estimates across transformer layers. The resulting walk state $R^k$ is updated at layer $k$ as:
\begin{align*}
    R^0 \;=\; (\hat{\mathbf{A}}_{\text{block}}^0)^{s}, \qquad
    R^k \;=\; R^{k-1} (\hat{\mathbf{A}}_{\text{block}}^k)^{s}, \quad k > 0,
\end{align*}
where $\hat{\mathbf{A}}_{\text{block}}^k$ denotes the sketched block attention matrix at layer $k$, and $s$ is a sparsity exponent.

The sketch-determined walk state $R^k$ aggregates block-level attention estimates from the current layer with those propagated from preceding layers, yielding a block-level score matrix that encodes \emph{inter-layer relationships} among tokens and reflects \emph{block-wise groupings} of blocks according to their layerwise attention relationships. The top-$\tau$ key blocks for sparse attention are then selected according to the corresponding row of $R^k$.

We implement custom Triton kernels for sparse attention in both the prefill and decode phases to efficiently execute attention over the selected blocks. Algorithms~\ref{alg:sparse_prefill} and~\ref{alg:sparse_decode} provide a detailed walkthrough of \textbf{Sketch\&Walk Attention} for the prefill and decode phases, respectively.

\begin{table*}[t]
\centering
\caption{Prefill phase performance comparison on LongBench. AVG$^{\overline{pc}}$ excludes PassageCount.}
\label{tab:longbench_prefill_llama31_8b}
\resizebox{\linewidth}{!}{
\begin{tabular}{lcccccccccccccccccc}
% \toprule
% \textbf{Method}
 & \rot{2wikimqa}
 & \rot{gov-report}
 & \rot{hotpot-qa}
 & \rot{lcc}
 & \rot{multifieldqa-en}
 & \rot{multinews}
 & \rot{musique}
 & \rot{narrativeqa}
 & \rot{passage-count}
 & \rot{passage-retrieval}
 & \rot{qasper}
 & \rot{qmsum}
 & \rot{repobench-p}
 & \rot{samsum}
 & \rot{trec}
 & \rot{triviaqa}
 & \rot{AVG}
 & \rot{AVG$^{\overline{pc}}$} \\
\toprule

\multicolumn{18}{c}{\textbf{Llama3.1-8B-Instruct}} \\

\cmidrule{2 - 19}

Dense
& 45.62 & 34.77 & 55.40 & 55.13 & 55.97 & 26.90 & 29.41 & 30.05
& 10.00 & 99.00 & 44.67 & 25.14 & 47.79 & 43.24 & 73.00 & 91.16
& 47.95 & 50.48 \\

\cmidrule{2-19}

FlexPrefill
& 34.25 & 33.52 & 54.86 & 54.11 & 54.23 & 26.71 & 28.95 & 27.62
& 1.00 & 63.50 & 39.21 & \textbf{25.97} & \textbf{54.65} & 43.26 & 71.50 & 89.00
& 43.90 & 46.76 \\

MInference
& 46.39 & 33.76 & 54.18 & \textbf{55.43} & 54.75 & 26.64 & 27.18 & 28.60
& 4.25 & 97.50 & 43.30 & 25.62 & 50.97 & 43.27 & \textbf{72.50} & 90.40
& 47.17 & 50.03 \\

Sketch\&Walk
& \textbf{47.26} & \textbf{34.39} & \textbf{56.69} & 52.76 & \textbf{55.30} & \textbf{26.92} & \textbf{31.27} & \textbf{29.61}
& \textbf{6.16} & \textbf{99.00} & \textbf{44.13} & 25.26 & 47.88 & \textbf{43.99} & 70.00 & \textbf{92.02}
& \textbf{47.67} & \textbf{50.43} \\

\cmidrule{2 - 19}

\multicolumn{18}{c}{\textbf{Llama3.2-1B-Instruct}} \\

\cmidrule{2 - 19}

Dense
& 29.34 & 29.49 & 30.26 & 30.29 & 41.29 & 25.96 & 14.61 & 18.59
& 3.14  & 5.00  & 21.05 & 21.46 & 25.44 & 39.26 & 62.00 & 78.89
& 29.75 & 31.53 \\

\cmidrule{2 - 19}

FlexPrefill
& 21.76 & 29.24 & 29.97 & 29.62 & 36.97 & 25.33 & 13.34 & 16.38
& \textbf{4.64}  & 4.50  & 16.18 & 21.24 & 27.44 & 38.24 & 58.00 & 74.55
& 27.96 & 29.52 \\

MInference
& 29.40 & 29.56 & 30.46 & 29.59 & 40.46 & 25.06 & 14.44 & 16.70
& 1.14  & \textbf{5.00}  & \textbf{20.44} & 21.64 & 26.81 & \textbf{39.42} & \textbf{61.50} & 76.21
& 29.24 & 31.11 \\

Sketch\&Walk
& \textbf{29.96} & \textbf{29.68} & \textbf{31.08} & \textbf{29.58}
& \textbf{41.49} & \textbf{25.73} & \textbf{15.54} & \textbf{18.67}
& 2.64 & \textbf{5.00} & \textbf{20.44} & \textbf{21.82}
& \textbf{33.54} & 39.04 & 60.00 & \textbf{77.59}
& \textbf{30.11} & \textbf{31.94} \\

\cmidrule{2 - 19}

\multicolumn{19}{c}{\textbf{Qwen2-7B-Instruct}} \\
\cmidrule{2-19}

Dense
& 43.71 & 32.12 & 43.99 & 48.64 & 47.31 & 24.39 & 25.26 & 23.82
& 5.50 & 69.00 & 46.74 & 22.95 & 47.66 & 33.59 & 55.00 & 83.97
& 40.85 & 43.21 \\
\cmidrule{2-19}

FlexPrefill
& \textbf{43.77} & 32.22 & \textbf{43.45} & 47.17 & 46.73 & 24.61 & 24.48 & 25.20
& 5.50 & 61.00 & \textbf{47.47} & 23.19 & 47.12 & 34.22 & 61.00 & 85.68
& 40.80 & 43.15 \\

\methodname{}
& 42.67 & \textbf{33.86} & 42.49 & \textbf{48.81} & \textbf{47.29} & \textbf{25.87} & \textbf{24.92} & \textbf{25.27}
& 5.50 & \textbf{67.00} & 45.66 & \textbf{23.92} & \textbf{55.61} & \textbf{46.31} & \textbf{77.00} & \textbf{89.07}
& \textbf{43.83} & \textbf{46.38} \\

\bottomrule
\end{tabular}
}
\end{table*}
% \vspace{-1.5mm}

\subsection{Error Bounds and Approximation Analysis}
We analyze the approximation properties of \methodname{} Attention and establish provable guarantees for both block selection and sparse attention performance. We show that \methodname{} provides accurate identification of important attention blocks, and the resulting sparse attention closely approximates dense attention. We begin by formalizing assumptions on token representations and attention distributions:

\begin{assumption}[Block Coherence]
\label{ass:block_coherence}
Within each block $\mathcal{B}_i$, token embeddings exhibit \textbf{semantic coherence}: tokens within a block share similar contextual representations. Formally, for query block $i$, each token $\mathbf{q}_t^{(i)}$ can be decomposed as:
\begin{align*}
    \mathbf{q}_t^{(i)} = \boldsymbol{\mu}_i + \boldsymbol{\epsilon}_t, \quad \text{where } \mathbb{E}[\boldsymbol{\epsilon}_t] = \mathbf{0}, \; \text{Var}(\boldsymbol{\epsilon}_t) \leq \sigma^2 \mathbf{I}
\end{align*}
with intra-block variance $\sigma^2 \ll \|\boldsymbol{\mu}_i\|^2$. This assumption is empirically justified by the observation that consecutive tokens in natural language often belong to the same semantic unit (phrase, clause, or sentence), and positional encodings induce smooth variations within local windows.
\end{assumption}

\begin{assumption}[Heavy-Tailed Block Attention Distribution]
\label{ass:heavy_tail}
The distribution of block-level attention scores follows a heavy-tailed pattern: for each query block $i$, there exists a small subset $\mathcal{S}_i^* \subset \{1, \ldots, b\}$ with $|\mathcal{S}_i^*| = \tau \ll b$ such that:
$
    \sum_{j \in \mathcal{S}_i^*} \text{softmax}(A_{ij}^{\text{true}}) \geq 1 - \eta,
$
for small $\eta > 0$ (typically $\eta < 0.1$). This reflects the finding that attention in LLMs concentrates on semantically relevant positions, with most blocks receiving negligible attention mass.
\end{assumption}

Under these assumptions, we show that token-space sketching via block aggregation acts as a subspace embedding, yielding \emph{accurate block-level representations with high probability}. Meanwhile, feature-space sketching via the Subsampled Randomized Hadamard Transform \emph{preserves block-level inner products up to small error} when the sketch dimension is sufficiently large.

\begin{restatable}{lemma}{subspaceEmbedding}\label{lem:subspace_embedding}
\textnormal{(Token-space Sketching: Subspace Embedding via Block Averaging).}
Under Assumption~\ref{ass:block_coherence}, the block average $\bar{\mathbf{q}}_i = \frac{1}{B}\sum_{t=1}^{B} \mathbf{q}_t^{(i)}$ satisfies:
\begin{align*}
    \|\bar{\mathbf{q}}_i - \boldsymbol{\mu}_i\|_2 \leq \frac{\sigma\sqrt{d \log(1/\delta)}}{\sqrt{B}}
\end{align*}
with probability at least $1 - \delta$.
\end{restatable}

\begin{restatable}{theorem}
{hadamardSketch}\label{thm:hadamard_sketch}
\textnormal{(Feature-space Sketching: Hadamard Transform Guarantee).}
Let $\mathbf{H}_d \in \mathbb{R}^{d \times k}$ be the Subsampled Randomized Hadamard Transform (SRHT). For any fixed vectors $\mathbf{u}, \mathbf{v} \in \mathbb{R}^d$ and $\epsilon \in (0, 1)$, if $k = \Omega(\epsilon^{-2} \log(b/\delta))$ where $b$ is the number of blocks, then with probability at least $1 - \delta$:
\begin{align*}
    \left| (\mathbf{u}^\top \mathbf{H}_d)(\mathbf{H}_d^\top \mathbf{v}) - \mathbf{u}^\top \mathbf{v} \right| \leq \epsilon \|\mathbf{u}\|_2 \|\mathbf{v}\|_2
\end{align*}
\end{restatable}

Given sufficient sketch dimension and block size, the top-$\tau$ blocks selected using sketched attention scores coincide with those selected by full attention with high probability:

\begin{restatable}{lemma}
{ddsSoftmax}\label{lem:dds_softmax}
\textnormal{(SWS Top-$\tau$ Selection under Softmax).}
Under Assumptions~\ref{ass:block_coherence} and \ref{ass:heavy_tail}, let $\mathcal{S}^*$ be the true top-$\tau$ blocks based on $\text{softmax}(A_{ij}^{\text{true}}/\sqrt{d})$, and let $\hat{\mathcal{S}}$ be the blocks selected by Sketch\&Walk using $\text{softmax}(\hat{A}_{ij}^{\text{SWS}}/\sqrt{k})$. If the sketching dimension satisfies:
\begin{align*}
    k \geq \frac{C \log(b/\delta)}{\gamma^2} \cdot \max_{i,j} \|\bar{\mathbf{q}}_i\|^2 \|\bar{\mathbf{k}}_j\|^2
\end{align*}
and the block size satisfies $B \geq C' \sigma^2 d \log(1/\delta) / \gamma^2$, then $\mathcal{S}^* = \hat{\mathcal{S}}$ with probability at least $1 - 2\delta$.
\end{restatable}

Finally, under the heavy-tailed attention assumption, restricting attention to the selected top-$\tau$ blocks yields an output that closely approximates dense attention, with error controlled by the tail mass and the quality of token-space aggregation:

\begin{restatable}{lemma}{outputApprox}
\label{lem:output_approx}
\textnormal{(Attention Output Approximation).}
Under Assumptions~\ref{ass:block_coherence}--\ref{ass:heavy_tail}, let $\mathbf{O}^{\text{full}}$ be the output of full attention and $\mathbf{O}^{\text{Sketch\&Walk}}$ be the output using Sketch\&Walk sparse attention with top-$\tau$ blocks. Then:
\begin{align}
    \|\mathbf{O}^{\text{Sketch\&Walk}} - \mathbf{O}^{\text{full}}\|_F &\leq \eta \|\mathbf{V}\|_F + O\left(\frac{\tau \sigma}{\sqrt{B}}\right) \|\mathbf{V}\|_2
\end{align}
where $\eta$ is the tail mass from Assumption~\ref{ass:heavy_tail}.
\end{restatable}

Detailed proofs of the above lemmas and theorems are provided in Appendix~\ref{app:error_bounds}.

%%%%%%%%%%%%%%%%% MOTIVATION
\section{Why \methodname{} for Sparse Attention?}

Many existing sparse attention methods such as \citep{tang2024quest,yan2025adamas,lai2025flexprefill} select key blocks using layer-wise relevance scores that approximate direct attention. This design introduces two implicit assumptions: (i) that \emph{block relevance can be determined from direct attention scores}, and (ii) that \emph{block selection can be performed independently at each layer}. While this might be effective in some tasks, we believe these assumptions break down in deep transformer architecture, i.e. LLMs.

From a theoretical perspective, the failure of direct-score selection follows from the properties of inner-product–based relevance measures. Attention scores are defined by inner products and therefore induce a non-metric similarity measure. As shown in Remark~\ref{rem:non_metric}, the inner product violates the triangle inequality. Thus, the existence of a block sequence $i \!\to\! j \!\to\! k$ with large pairwise relevance scores does not imply a large direct relevance score between $i$ and $k$. Under per-layer block selection, a query block $i$ will not select a key block $k$ when their relevance is expressed through intermediate blocks.
Such relevance becomes apparent only after attention is accumulated in deeper layers~\citep{cai2024pyramidkv,abnar2020quantifying}.
However, block $k$ was excluded from sparse attention in earlier layers, so this dependency cannot be recovered when it later becomes important.

\begin{wraptable}{r}{0.55\columnwidth}
\vspace{-1.0em}
\centering
\caption{End-to-End performance of Sketch\&Walk on RULER.}
\label{tab:ruler_end2end}
\resizebox{0.55\columnwidth}{!}{
\begin{tabular}{llccccc|c}

\toprule
\multirow{2}{*}{\textbf{Model}} & \multirow{2}{*}{\textbf{Method}} & 
\multicolumn{6}{c}{\textbf{Context Length}} \\
\cmidrule{3 - 8}
& & 4K & 8K & 16K & 32K & 64K & Average \\

\midrule

\multirow{2}{*}{\textbf{L3.1 8B}} 
& Dense        & 95.31 & 95.31 & 94.06 & 89.69 & 82.81 & 91.44 \\ 
& Sketch\&Walk & 93.64 & 95.00 & 93.13 & 87.50 & 81.50 & 90.15 \\ 

\bottomrule
\end{tabular}
}
\vspace{-0.8em}
\end{wraptable}

This limitation is also evident empirically. We visualize attention matrices from a layer of the Llama-3.1-8B-Instruct model~\cite{llama31} in Figure~\ref{fig:random_walk}. The matrix $A^1$ corresponds to direct attention scores, which can be viewed as an oracle relevance metric commonly used by existing sparse attention methods for top-$k$ block selection. While $A^1$ captures only immediate interactions, higher-order compositions such as $A^4$ and $A^7$ gradually reveal structure that emerges through repeated composition of attention operators. Blocks that appear weakly connected under $A^1$ can receive substantial mass under higher powers, \emph{indicating their relevance in later layers}.

These observations highlight a fundamental gap: \emph{block relevance is not solely determined by the current layer’s attention scores}. Instead, relevance emerges through the composition of attention across layers, and block selection for sparse attention must account for this accumulation in order to preserve information required by subsequent layers. This motivates selecting blocks based on higher-order attention structure rather than per-layer scores.

Sketch\&Walk addresses this issue by explicitly accumulating attention across layers. At each layer, we estimate a block-level attention matrix $\hat A^k_{\mathrm{block}}$ using Small-World Sketching and maintain a Sketch-Determined Walk state

\begin{equation}
R_0 = (\hat A^{0}_{\mathrm{block}})^s, \qquad
R_k = R_{k-1} (\hat A^{k}_{\mathrm{block}})^s,
\end{equation}
which captures higher-order compositions of attention. As formalized in Lemma~B.12, selecting top-$\tau$ blocks based on $R^k[i,j]$ identifies blocks whose importance emerges consistently across layers, rather than those that are locally salient in a single layer. These blocks minimize reconstruction error with respect to the full transformer attention pattern.

By selecting blocks from the accumulated relevance scores in $R^k$, \methodname{} mitigates the limitations of layer-wise selection and enables sparse attention to better approximate dense multi-layer attention behavior.

%%%%%%%%%%%%%%%%%% EXPERIMENTS
\begin{table*}[t]
\centering
\caption{Decoding phase performance comparison on LongBench. AVG$^{\overline{pc}}$ excludes PassageCount.}
\label{tab:longbench_decode_llama31_8b}
\resizebox{\linewidth}{!}{
\begin{tabular}{lcccccccccccccccccc}
 & \rot{2wikimqa}
 & \rot{govreport}
 & \rot{hotpotqa}
 & \rot{lcc}
 & \rot{multifieldqa-en}
 & \rot{multinews}
 & \rot{musique}
 & \rot{narrativeqa}
 & \rot{passage-count}
 & \rot{passage-retrieval}
 & \rot{qasper}
 & \rot{qmsum}
 & \rot{repobench-p}
 & \rot{samsum}
 & \rot{trec}
 & \rot{triviaqa}
 & \rot{AVG}
 & \rot{AVG$^{\overline{pc}}$} \\
\toprule

\multicolumn{18}{c}{\textbf{Llama3.1-8B-Instruct}} \\

\cmidrule{2-19}

Dense
& 45.62 & 34.77 & 55.40 & 55.13 & 55.97 & 26.90 & 29.41 & 30.05
& 10.00 & 99.00 & 44.67 & 25.14 & 47.79 & 43.24 & 73.00 & 91.16
& 47.95 & 50.48 \\

\cmidrule{2-19}

Quest
& 46.95 & \textbf{34.80} & 55.93 & 49.96 & \textbf{55.98} & \textbf{27.32} & 30.89 & 28.79
& 8.03 & 99.00 & 42.21 & 24.19 & 45.12 & 42.64 & 69.00 & 89.90
& 46.92 & 49.51 \\

Adamas
& 45.81 & 34.72 & 56.07 & 54.20 & 56.20 & 26.89 & 31.20 & 30.09
& 8.00 & 99.00 & 43.84 & 25.05 & 46.96 & 43.14 & \textbf{73.00} & 91.29
& 47.84 & 50.49 \\

Sketch\&Walk
& \textbf{47.11} & 32.39 & \textbf{57.49} & \textbf{54.94} & 55.73 & 25.55 & \textbf{33.16} & \textbf{30.69}
& \textbf{8.75} & 99.00 & \textbf{44.44} & \textbf{25.47} & \textbf{47.78} & \textbf{43.39} & 70.50 & \textbf{91.40}
& \textbf{48.00} & \textbf{50.60} \\

\cmidrule{2-19}

\multicolumn{18}{c}{\textbf{Llama3.2-1B-Instruct}} \\

\cmidrule{2-19}

Dense
& 29.34 & 29.49 & 30.26 & 30.29 & 41.29 & 25.96 & 14.61 & 18.59
& 3.14  & 5.00  & 21.05 & 21.46 & 25.44 & 39.26 & 62.00 & 78.89
& 29.75 & 31.53 \\

\cmidrule{2-19}

Quest
& 28.93 & 26.61 & 30.14 & 31.50 & 38.95 & 24.54 & 14.51 & 17.18
& 2.50  & 5.00  & 19.79 & 20.39 & 33.48 & 38.37 & 60.00 & 71.58
& 28.97 & 30.73 \\

Adamas
& 29.45 & \textbf{29.41} & 29.06 & \textbf{30.70} & 40.08 & 23.86 & 15.10 & 17.94
& \textbf{3.14}  & 5.00  & 19.34 & 20.95 & 34.03 & 37.63 & \textbf{62.00} & 77.89
& 29.72 & 31.50 \\

Sketch\&Walk 
& \textbf{30.14} & 28.11 & \textbf{30.79} & 29.35 & \textbf{40.38}
& \textbf{24.91} & \textbf{15.54} & \textbf{19.30}
& \textbf{3.14} & \textbf{5.75} & \textbf{20.18} & \textbf{21.78}
& \textbf{36.02} & \textbf{38.79} & 61.50 & \textbf{79.81}
& \textbf{30.34} & \textbf{32.16} \\

\bottomrule
\end{tabular}
}
\end{table*}

\section{Experiments}\label{sec:experiments}
%%%%%%%%%%%% Settings %%%%%%%%%%%%%%%%
\subsection{Settings}
\textbf{Models.}
To maximize coverage of models supported by all baselines, we evaluate Sketch\&Walk on widely used long-context instruction-tuned models of different scales: Llama-3.1-8B-Instruct \citep{llama31}, Llama-3.2-1B-Instruct \citep{llama32}, and Qwen2-7B-Instruct~\cite{qwen2}. All model supports context lengths up to 128K tokens. We use the default chat template for all instruct models in subsequent experiments. Following \citet{tang2024quest}, we do not apply Sketch\&Walk to the first two layers, as these layers typically exhibit low achievable sparsity.

\textbf{Datasets.}
We evaluate methods on two benchmarks that pose complementary challenges for long-context understanding: 
(i) LongBench \citep{bai2024longbench}, a diverse benchmark covering question answering, reasoning, summarization, and code understanding tasks with input lengths up to tens of thousands of tokens; and 
(ii) RULER \citep{hsieh2024ruler}, a synthetic diagnostic benchmark designed to stress-test a model’s ability to retrieve and utilize sparse, position-sensitive information within very long contexts. 

\begin{wraptable}{r}{0.55\columnwidth}
\vspace{-1.0em}
\centering
\caption{Prefill phase performance comparison on RULER.}
\label{tab:ruler_prefill}
\resizebox{0.55\columnwidth}{!}{
\begin{tabular}{llccccc|c}

\toprule
\multirow{2}{*}{\textbf{Model}} & \multirow{2}{*}{\textbf{Method}} & \multicolumn{6}{c}{\textbf{Context Length}} \\
\cmidrule{3 - 8}
& & 4K & 8K & 16K & 32K & 64K & Average \\

\midrule

\multirow{4}{*}{\textbf{L3.1 8B}} 
& Dense        & 95.31 & 95.31 & 94.06 & 89.69 & 82.81 & 91.44 \\ 
& FlexPrefill  & 95.31 & 92.50 & 93.75 & 88.13 & 80.31 & 90.00 \\
& MInference   & 92.65 & \textbf{95.00} & 93.12 & 86.87 & \textbf{82.50} & 90.03 \\ 
& Sketch\&Walk & \textbf{96.56} & 94.06 & \textbf{94.06} & \textbf{89.63} & \textbf{82.50} & \textbf{91.35} \\ 

\midrule

\multirow{4}{*}{\textbf{L3.2 1B}} 
& Dense        & 74.38 & 68.44 & 65.31 & 66.88 & 66.56 & 68.31 \\ 
& FlexPrefill  & \textbf{74.06} & 63.75 & 60.63 & 61.88 & 55.94 & 63.25 \\ 
& MInference   & 73.44 & \textbf{69.06} & 64.69 & 64.38 & 62.71 & 66.86 \\ 
& Sketch\&Walk & \textbf{74.06} & 68.44 & \textbf{66.06} & \textbf{65.62} & \textbf{63.75} & \textbf{67.59} \\ 

\bottomrule
\end{tabular}
}
\vspace{-0.8em}
\end{wraptable}

\textbf{Implementation Details.}
All experiments are conducted on a single NVIDIA H100 GPU with 94\,GB of memory. We implement a custom inference pipeline in PyTorch to support efficient attention computation over long-context inputs. Our implementation leverages Triton \citep{tillet2019triton} to optimize GPU kernel performance and uses a block size of 64 tokens across all experiments. Following common practice in block-sparse attention, we always retain the first and last key blocks for each query block. All evaluations are performed using greedy decoding to ensure result consistency. Unless otherwise specified, \methodname{} uses a sketching dimension of 64 and sparsity exponent of 8.

\textbf{Baselines.}
We compare Sketch\&Walk against dense attention and state-of-the-art sparse attention methods. For prefill optimization, we evaluate against MInference \citep{jiang2024minference} and FlexPrefill \citep{lai2025flexprefill}. For decode optimization, we compare against QUEST \citep{tang2024quest} and Adamas \citep{yan2025adamas}. For each baseline, we adopt the best configuration reported in its respective paper at the same sparsity level to ensure a fair comparison.

%%%%%%%%%%%% Accuracy %%%%%%%%%%%%%%%%
\subsection{Accuracy Evaluation}

\begin{wraptable}{r}{0.55\columnwidth}
\vspace{-1.0em}
\centering
\caption{Decoding phase performance comparison on RULER.}
\label{tab:ruler_decode}
\resizebox{0.55\columnwidth}{!}{
\begin{tabular}{llccccc|c}

\toprule
\multirow{2}{*}{\textbf{Model}} & \multirow{2}{*}{\textbf{Method}} & \multicolumn{6}{c}{\textbf{Context Length}} \\
\cmidrule{3 - 8}
& & 4K & 8K & 16K & 32K & 64K & Average \\

\midrule

\multirow{4}{*}{\textbf{L3.1 8B}} 
& Dense        & 95.31 & 95.31 & 94.06 & 89.69 & 82.81 & 91.44 \\ 
& Quest        & 91.88 & 92.81 & 92.81 & 87.50 & 80.31 & 89.06 \\ 
& Adamas       & 95.31 & 94.06 & \textbf{94.06} & \textbf{89.06} & 81.19 & 90.74 \\ 
& Sketch\&Walk & \textbf{95.94} & \textbf{96.19} & 93.13 & 88.63 & \textbf{81.50} & \textbf{91.08} \\ 

\bottomrule
\end{tabular}
}
\vspace{-0.8em}
\end{wraptable}

\textbf{Prefill Phase Comparison.}
Tables~\ref{tab:longbench_prefill_llama31_8b} and~\ref{tab:ruler_prefill} report the prefill-phase performance of \methodname{} compared against prefill-optimized sparse attention methods on LongBench and RULER, respectively. Evaluations are conducted at an 80\% attention sparsity level. Across both benchmarks, spanning short to long contexts (4K to 64K tokens) on RULER and diverse real-world tasks on LongBench, \methodname{} achieves the best overall performance among all compared methods. Notably, \methodname{} consistently preserves model accuracy across a wide range of context lengths and, in some settings, even improves performance (e.g. about 3\% on overall performance of Qwen2 on LongBench) while accelerating prefill computation.

\textbf{Decode Phase Comparison.}
Tables~\ref{tab:longbench_decode_llama31_8b} and~\ref{tab:ruler_decode} report the decode-phase performance of \methodname{} compared against decode-optimized sparse attention methods on LongBench and RULER, respectively, under an 80\% attention sparsity setting. Across both benchmarks, \methodname{} achieves competitive performance relative to existing baselines, closely matching dense attention and, in some cases, surpassing it. For example, we observe up to a 3\% accuracy improvement on the \textsc{musique} dataset.

\textbf{End-to-End Performance.}
Tables~\ref{tab:longbench_prefill_decode_llama31_8b} and~\ref{tab:ruler_end2end} show the end-to-end performance of \methodname{} on LongBench and RULER, respectively, when applied end-to-end across both the prefill and decode stages. On LongBench, \methodname{} closely matches dense attention across a wide range of tasks, with comparable average accuracy and minimal degradation under an 80\% sparsity setting. On RULER, \methodname{} maintains strong performance across context lengths from 4K to 64K tokens, exhibiting only a modest average accuracy drop relative to dense attention. Overall, these results demonstrate that \methodname{} preserves end-to-end model quality while benefiting from sparsity-induced efficiency gains throughout both phases of the full inference process.

\begin{table*}[t]
\centering
\caption{End-to-End performance of Sketch\&Walk on LongBench. Sketch\&Walk is evaluated at 80\% sparsity level. AVG$^{\overline{pc}}$ excludes PassageCount.}
\label{tab:longbench_prefill_decode_llama31_8b}
\resizebox{\linewidth}{!}{
\begin{tabular}{lcccccccccccccccccc}
 & \rot{2wiki\\mqa}
 & \rot{gov\\report}
 & \rot{hotpot\\qa}
 & \rot{lcc}
 & \rot{multifieldqa\\en}
 & \rot{multi\\news}
 & \rot{musique}
 & \rot{narrative\\qa}
 & \rot{passage\\count}
 & \rot{passage\\retrieval}
 & \rot{qasper}
 & \rot{qmsum}
 & \rot{repobench}
 & \rot{samsum}
 & \rot{trec}
 & \rot{trivia\\qa}
 & \rot{AVG}
 & \rot{AVG$^{\overline{pc}}$} \\
\toprule

\multicolumn{19}{c}{\textbf{Llama3.1-8B-Instruct}} \\

\cmidrule{2-19}

Dense
& 45.62 & 34.77 & 55.40 & 55.13 & 55.97 & 26.90 & 29.41 & 30.05
& 10.00 & 99.00 & 44.67 & 25.14 & 47.79 & 43.24 & 73.00 & 91.16
& 47.95 & 50.48 \\

Sketch\&Walk
& 45.54 & 33.22 & 56.51 & 53.78 & 54.53 & 25.96 & 33.11 & 31.18
& 6.53 & 98.00 & 44.14 & 25.17 & 48.63 & 43.69 & 70.00 & 92.08
& 47.63 & 50.37 \\

\cmidrule{2-19}

\multicolumn{19}{c}{\textbf{Llama3.2-1B-Instruct}} \\

\cmidrule{2-19}

Dense
& 29.34 & 29.49 & 30.26 & 30.29 & 41.29 & 25.96 & 14.61 & 18.59
& 3.14  & 5.00  & 21.05 & 21.46 & 25.44 & 39.26 & 62.00 & 78.89
& 29.75 & 31.53 \\

Sketch\&Walk
& 29.54 & 28.34 & 31.01 & 29.35 & 40.97 & 24.61 & 14.11 & 17.98
& 3.00  & 5.30  & 19.67 & 21.56 & 36.02 & 38.79 & 61.50 & 79.81
& 30.10 & 31.90 \\

\cmidrule{2 - 19}

\multicolumn{19}{c}{\textbf{Qwen2-7B-Instruct}} \\
\cmidrule{2-19}

Dense
& 43.71 & 32.12 & 43.99 & 48.64 & 47.31 & 24.39 & 25.26 & 23.82
& 5.50 & 69.00 & 46.74 & 22.95 & 47.66 & 33.59 & 55.00 & 83.97
& 40.85 & 43.21 \\

\methodname{}
& 42.14 & 31.59 & 42.55 & 48.27 & 46.92 & 24.96 & 23.73 & 24.69
& 5.50 & 65.00 & 45.59 & 23.28 & 55.26 & 46.56 & 76.00 & 89.51
& 43.22 & 45.74 \\

\bottomrule
\end{tabular}
}
\end{table*}

%%%%%%%%%%%% Efficiency %%%%%%%%%%%%%%%%
\subsection{Efficiency Evaluation}

\begin{wrapfigure}{r}{0.48\columnwidth}
\vspace{-1.0em}
\centering
\includegraphics[width=0.48\columnwidth]{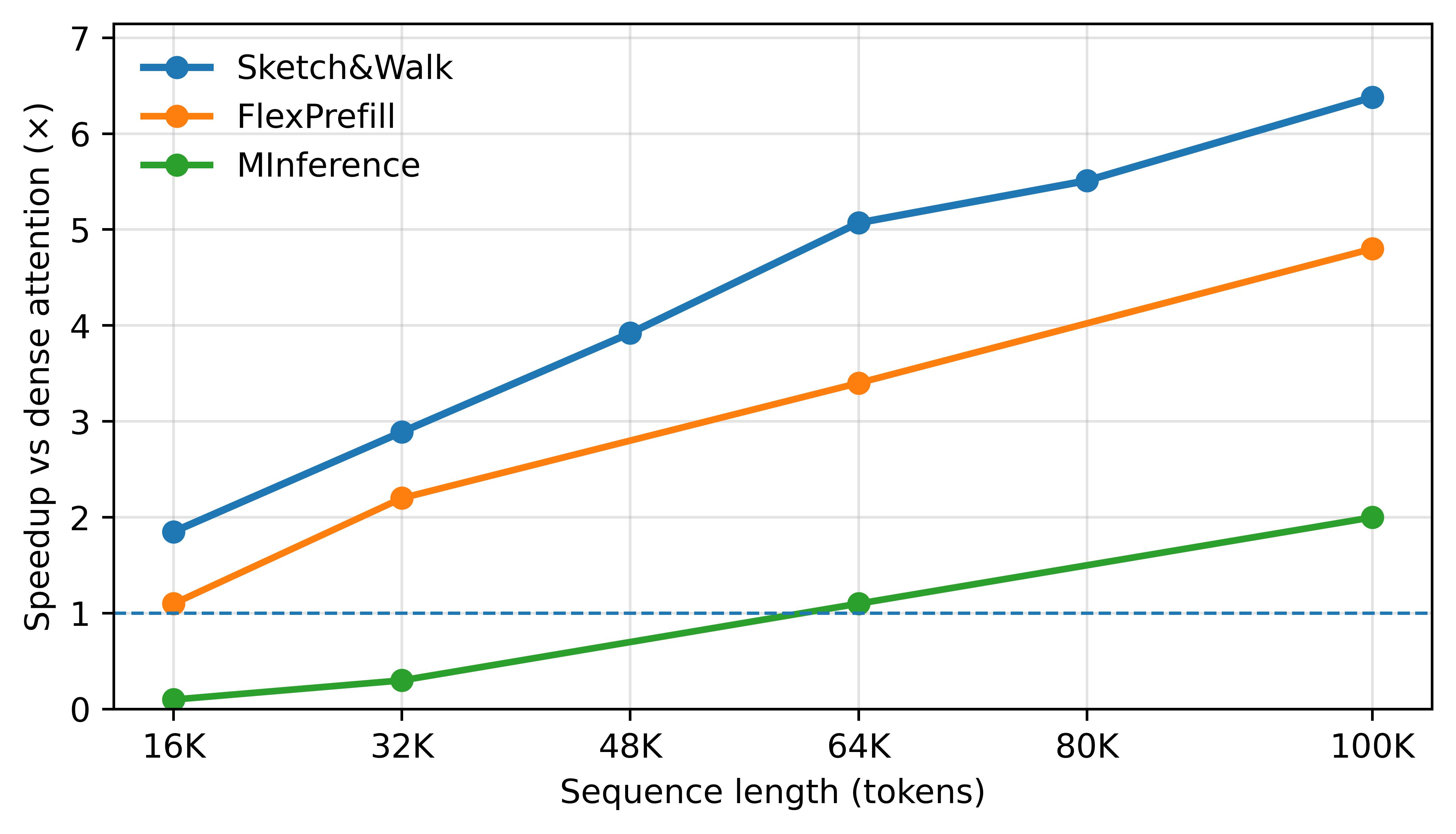}
\caption{Prefill Phase Acceleration}
\label{fig:prefill_acceleration}
\vspace{-0.8em}
\end{wrapfigure}

We evaluate the practical effectiveness of \methodname{} for inference acceleration across a diverse range of context lengths, from 16K to 128K tokens. In addition, we analyze the sparse attention kernel of \methodname{} to examine the overhead introduced by attention score estimation, and the sketch-determined walk operation.

\textbf{Prefill Phase Acceleration.}
In Figure~\ref{fig:prefill_acceleration}, we report the average speedup of each method relative to dense attention in prefill phase, computed from the ratio of time-to-first-token (TTFT) between the method and the dense attention, FlashAttention~\citep{dao2022flashattention}. We evaluate with Llama-3.1-8B-Instruct model under single-batch scenarios and a fixed sparsity level of 90\%. \methodname{} delivers strong and scalable acceleration, with speedups growing approximately linearly as context length increases. Unlike prior approaches such as MInference \citep{jiang2024minference}, which only become effective beyond sufficiently long contexts, \methodname{} provides consistent speedups across diverse sequence lengths. At the same sparsity level, \methodname{} also achieves the highest prefill-phase acceleration among all compared methods.

\textbf{Decode Phase Acceleration.}
In Figure~\ref{fig:decode_acceleration}, we report the average speedup of \methodname{} relative to dense attention during decoding, computed from the token throughput ratio between \methodname{} and the dense-attention baseline implemented with FlashAttention~\citep{dao2022flashattention}. We evaluate using the Llama-3.1-8B-Instruct model under a single-batch setting and a fixed sparsity level of 90\%. \methodname{} delivers consistent and scalable decoding acceleration, achieving up to a $1.6\times$ speedup.

\begin{wrapfigure}{r}{0.45\columnwidth}
\vspace{-1.0em}
\centering
\includegraphics[width=0.45\columnwidth]{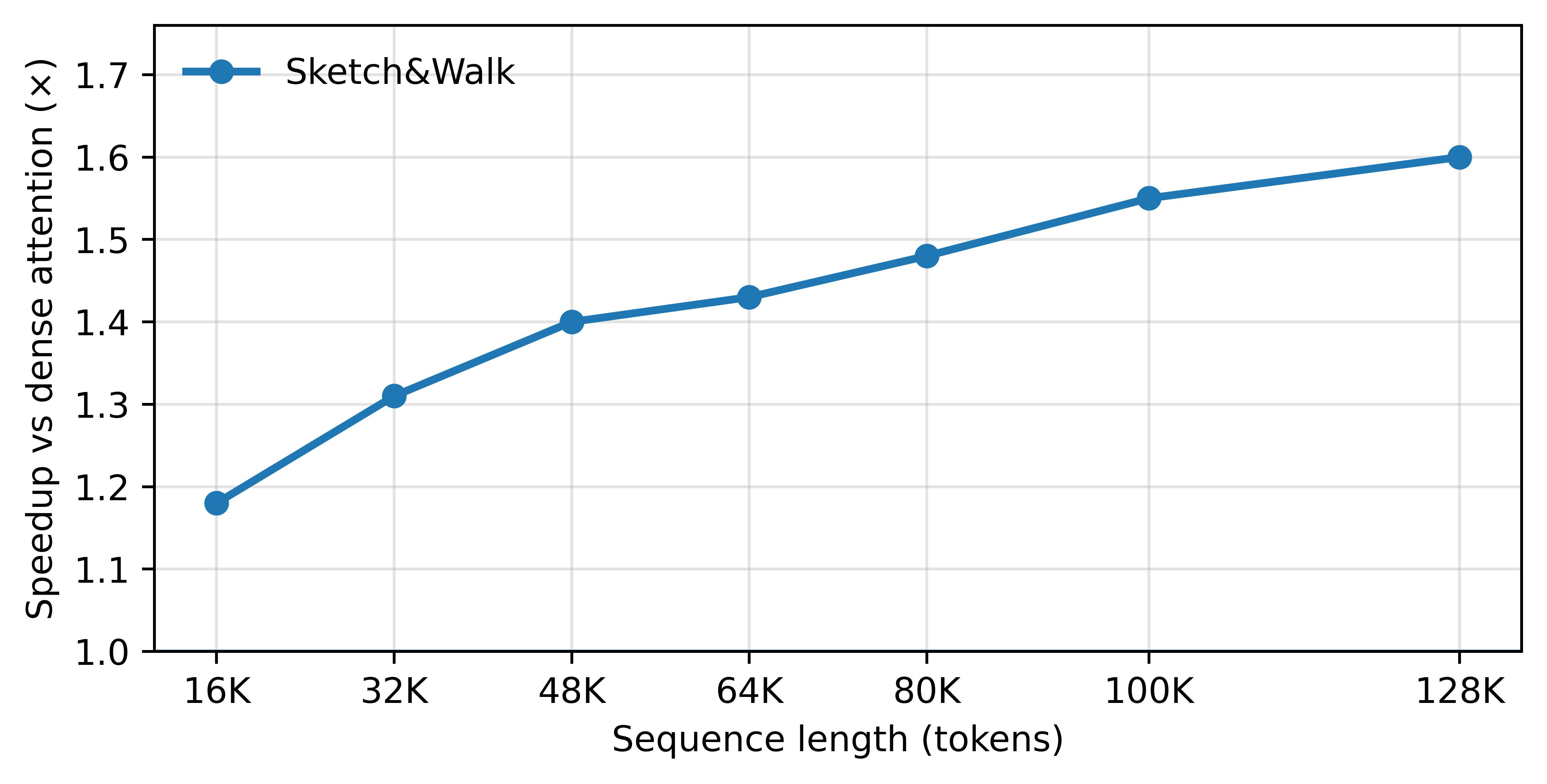}
\caption{Decode Phase Acceleration}
\label{fig:decode_acceleration}
\vspace{-0.8em}
\end{wrapfigure}

\textbf{Kernel Analysis.}
We analyze the computational overhead introduced by attention score estimation and the sketch-determined walk operation relative to the total cost of a sparse attention computation in \methodname{}, as well as to full dense attention. Evaluations are conducted across multiple context lengths under different fixed sparsity budgets. Our implementation employs a custom CUDA kernel for fast Hadamard sketching and efficient Triton kernels for sparse attention in both the prefill and decode phases. As shown in Figure~\ref{fig:kernel_analysis}, the overhead from Hadamard sketching and random-walk estimation contributes only a small fraction of the overall sparse attention cost. Sparse attention incurs significantly lower cost than dense attention, with the acceleration increasing as sequence length grows.

\begin{figure*}
    \centering
    \includegraphics[width=0.85\linewidth]{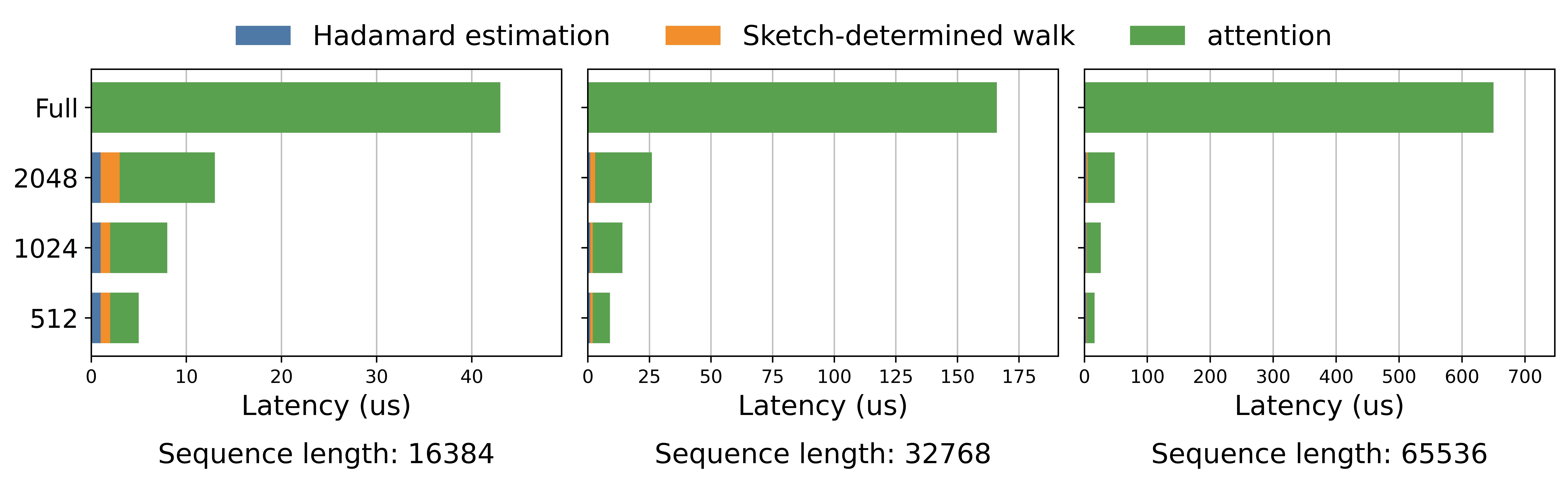}
    \caption{\methodname{} Sparse Attention Kernel Analysis}
    \label{fig:kernel_analysis}
    \vspace{-0.5em}
\end{figure*}

%%%%%%%%%%%% Ablation %%%%%%%%%%%%%%%%
\subsection{Ablation Studies}

\textbf{Robustness to sketch dimension.}
Figure~\ref{fig:ablation}(a) shows the effect of the Hadamard sketch dimension.
Among 3 datasets, \methodname{} exhibits consistently strong performance over a wide range of sketch sizes.
Even with small sketch dimensions (e.g., 16–32), accuracy remains close to that of larger sketch size.
\methodname{} remains effective under aggressive dimensionality reduction.

\textbf{Sensitivity to walk degree.}
Figure~\ref{fig:ablation}(b) studies the impact of the walk degree.
Performance improves with increasing walk degrees and stabilizes thereafter, suggesting that only limited attention composition is needed in practice. Notably, increasing the walk degree does not introduce instability or degrade performance, indicating that the Sketch-Determined Walk behaves reliably across a broad range of settings.

\textbf{Robustness under varying sparsity.}
Figure~\ref{fig:ablation}(c) reports results under different sparsity ratios. \methodname{} maintains near-lossless accuracy even at high sparsity levels.
\methodname{} remains effective in highly constrained inference regimes and can achieve favorable accuracy--efficiency trade-offs.

\begin{figure*}
    \centering
    \includegraphics[width=0.80\linewidth]{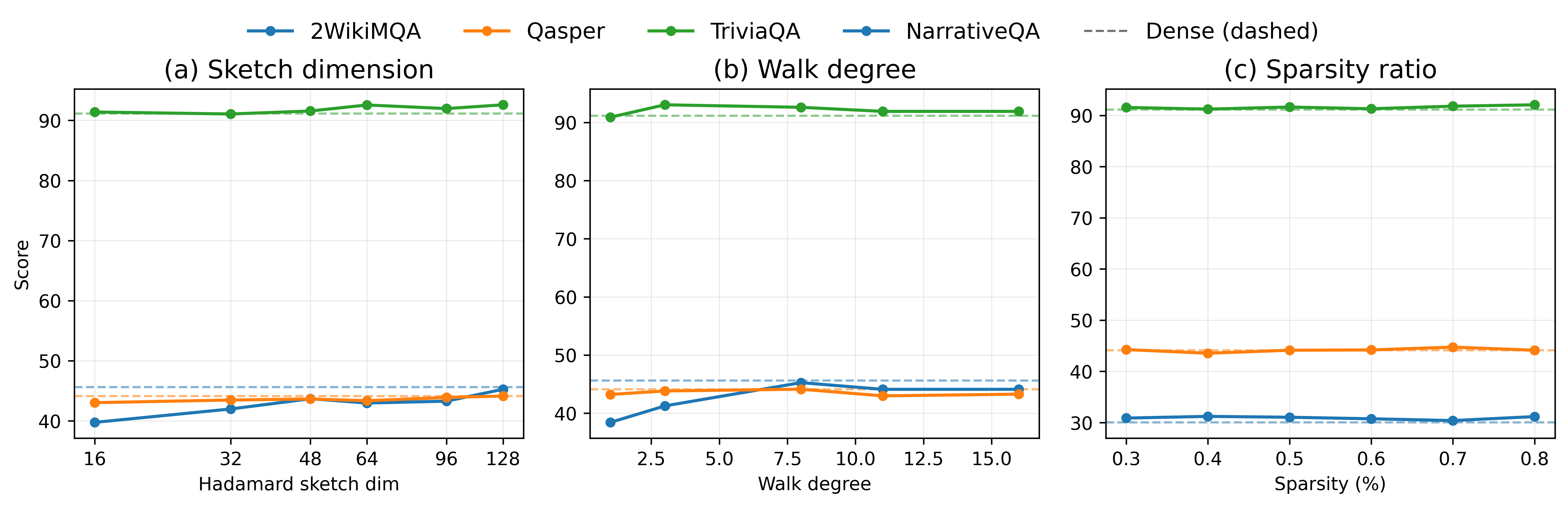}
    \caption{Ablation Studies on different parameters and sparsity level of \methodname{}}
    \label{fig:ablation}
    \vspace{-1em}
\end{figure*}

%%%%%%%%%%%%%%%%%%% RELATED WORKS
\section{Related Works}\label{sec:related_work}

A large body of work has addressed the computational and memory bottlenecks of attention in long-context settings.
These methods typically exploit natural sparsity in the attention mechanism, identified using inexpensive heuristics or approximations, to either reduce computation on less important token pairs or reduce the memory footprint of key--value (KV) caches.
Prior approaches often target a specific stage of inference (prefill or decode), rely on auxiliary training, or introduce architectural constraints.
In contrast, \methodname{} is a training-free sparse attention method that applies end-to-end to both prefill and decode phases.

\textbf{Sparse attention for prefilling.}
Several methods accelerate the prefill phase by restricting attention to a subset of interactions.
Static sparsity patterns, such as Sparse Transformer~\citep{child2019sparse_transformer}, Longformer~\citep{beltagy2020longformer}, and BigBird~\citep{zaheer2020bigbird}, reduce complexity using fixed local or block-based layouts.
Dynamic sparsity methods aim to adapt attention patterns to the input.
MInference~\citep{jiang2024minference} selects important interactions using pre-determined sparsity pattern, while FlexPrefill~\citep{lai2025flexprefill} leverages compiler-supported flexible block patterns.
Although effective, these approaches often incur nontrivial pre-computation or scheduling overhead.
Training-based methods such as SeerAttention~\citep{gao2024seerattention} introduce sparsity through learned gating mechanisms, improving efficiency at the cost of additional training and potentially reduced generalization.

\textbf{Efficient attention for decoding.}
Another line of work focuses on reducing memory and computation during decoding by dynamically selecting, pruning, or compressing the KV cache.
Methods such as H$_2$O~\citep{zhang2023h2o}, TOVA~\citep{oren2024transformers_msrnns}, and InfLLM~\citep{xiao2024infllm} discard tokens based on query-dependent importance criteria.
StreamingLLM~\citep{xiao2024streamingllm} retains only initial and recent tokens to achieve stable latency and memory usage.
More recent approaches, including Quest~\citep{tang2024quest} and Rectified Sparse Attention~\citep{sun2025resa}, adaptively select tokens to maintain accuracy under high sparsity.
These methods primarily target the decode phase.

Among these methods, Adamas~\citep{yan2025adamas} follows the close direction to \methodname{}.
Both approaches leverage the Hadamard transform to obtain lightweight representations of queries and keys.
However, Adamas uses the transformed representations directly as a distance-based criterion for token selection and operates in an independent, per-layer manner during decoding.
In contrast, \methodname{} uses sketched representations to estimate attention score and explicitly accounts for cross-layer composition through the Sketch-Determined Walk.
This design allows \methodname{} to capture dependencies that emerge across layers and to apply uniformly to both prefill and decode phases.

\textbf{Alternative attention architectures.}
Beyond sparse attention methods, several works propose alternative attention mechanisms or architectural changes.
These include sliding-window attention~\citep{beltagy2020longformer}, linear or gated attention~\citep{qiu2025gatedattention}, and state-space models~\citep{gu2023mamba}.
More recent designs such as Native Sparse Attention~\citep{yuan2025nsa} and DeepSeek Sparse Attention~\citep{deepseek2025v32} integrate sparsity at the architectural level.
While effective in some regimes, these approaches typically require model modification or retraining.
By contrast, \methodname{} is a post-training method that accelerates inference without architectural changes, proxy objectives, or additional training.

\textbf{Randomized Algorithms Drive Sparsity.}
Randomized algorithms, particularly hashing algorithms and their variants~\citep{broder1997resemblance,broder1998minwise,charikar2002similarity,ainr14,ar15,ailrs15,alrw17,li2019generalization,ll19_neurips,ll21_icml,ll21,li2023oporp,znv20+,zlw+18,pan2025alinfik}, have been instrumental in enabling sparsity-driven efficiency in machine learning systems. The MONGOOSE framework~\citep{chen2021mongoose} introduced learnable hashing for efficient neural network training by dynamically sparsifying gradient computations. Building on maximum inner product (MaxIP) data structures, \citet{xu2021breaking} broke the linear iteration cost barrier for sparse conditional gradient methods. \citet{guo2025zeroth} leveraged randomized empirical Fisher estimation to identify very sparse parameter subsets of LLMs, enabling memory-efficient zeroth-order fine-tuning with transferable static sparsity. More recently, Zen~\citep{wang2025zen} leverages hierarchical hashing to achieve sparsity-driven data synchronization for distributed training. However, randomized sparsity often introduces an accuracy-efficiency tradeoff~\citep{xu2025compression}. To address this, soft prompts have been shown to recover the performance loss of randomized algorithm-driven sparse LLMs in a transferable manner~\citep{xu2024soft}, and SIRIUS~\citep{zhou2024sirius} introduced contextual sparsity with correction mechanisms for efficient inference.

%%%%%%%%%%%%%%% CONCLUSION
\section{Conclusion}
We introduce \methodname{}, a training-free sparse attention method that applies uniformly to both the prefill and decode phases of LLM inference. By guiding sparse selection using sketched, walk-accumulated attention estimates across layers, \methodname{} achieves substantial inference speedups while preserving model quality. Empirically, \methodname{} maintains accuracy with minimal degradation, and in some cases yields improvements over dense attention, at an impressive 80\% sparsity. These results demonstrate that accounting for cross-layer attention composition enables effective and robust sparse attention for long-context inference.

\newpage
\appendix

%%%%%%%%%%%%%%%%%%%% APPENDIX THEORY

\section{Theoretical Analysis}
\label{sec:theory}

\subsection{Notation Summary}

For ease of reference, we summarize the key mathematical notation used throughout this section in Table~\ref{tab:notation}.

\begin{table}[h!]
\centering
\footnotesize
\caption{Summary of notation used in theoretical analysis.}
\label{tab:notation}
\begin{tabular}{@{}llll@{}}
\toprule
\textbf{Symbol} & \textbf{Description} & \textbf{Symbol} & \textbf{Description} \\
\midrule
\multicolumn{4}{@{}l}{\textit{Dimensions, Indices \& Basic Matrices}} \\
$n$ & Total number of tokens & $d$ & Head dimension \\
$B$ & Block size (tokens per block) & $b$ & Number of blocks, $b = \lceil n/B \rceil$ \\
$L$ & Number of transformer layers & $h$ & Number of attention heads \\
$k, r$ & Sketch dimension & $\tau$ & Top blocks for sparse attention \\
$\mathbf{Q}, \mathbf{K}, \mathbf{V}$ & Query, key, value matrices & $\mathbf{Q}^{(i)}, \mathbf{K}^{(i)}$ & Submatrices for block $i$ \\
$\bar{\mathbf{q}}_i, \bar{\mathbf{k}}_i$ & Block-averaged vectors & $\tilde{\mathbf{q}}_i, \tilde{\mathbf{k}}_i$ & Sketched block vectors \\
$\mathbf{q}_t^{(i)}$ & $t$-th token in query block $i$ & & \\
\midrule
\multicolumn{4}{@{}l}{\textit{Block Coherence Model \& Sketching Transforms}} \\
$\boldsymbol{\mu}_i$ & Block centroid for block $i$ & $\boldsymbol{\epsilon}_t$ & Noise for token $t$ \\
$\sigma^2$ & Intra-block variance bound & $\mathbf{H}_d$ & SRHT matrix \\
$\mathbf{D}$ & Diagonal Rademacher matrix & $\mathbf{W}$ & Walsh-Hadamard matrix \\
$\mathbf{S}$ & Coordinate sampling matrix & & \\
\midrule
\multicolumn{4}{@{}l}{\textit{Block Attention \& Sketch-Determined Walk}} \\
$A_{ij}^{\text{true}}$ & True block attention score & $\hat{A}_{ij}^{\text{SWS}}$ & Estimated block score (SWS) \\
$\hat{\mathbf{A}}_{\text{block}}^k$ & Sketched block attention at layer $k$ & $R^k$ & Walk state matrix at layer $k$ \\
$\mathbf{W}^k$ & $(\hat{\mathbf{A}}_{\text{block}}^k)^s$, exponentiated attention & $s$ & Sparsity exponent ($s > 1$) \\
\midrule
\multicolumn{4}{@{}l}{\textit{Block Selection Sets}} \\
$\mathcal{B}_i$ & The $i$-th block of tokens & $\mathcal{S}_i^*$ & True top-$\tau$ blocks for block $i$ \\
$\hat{\mathcal{S}}$ & Estimated top-$\tau$ blocks & $\mathcal{R}^L_i$ & Reachable blocks after $L$ layers \\
$\mathcal{C}_i$ & Consistently important blocks & & \\
\midrule
\multicolumn{4}{@{}l}{\textit{Attention Matrices \& Outputs}} \\
$\mathbf{A}^{\text{full}}$ & Full (dense) attention matrix & $\mathbf{A}^{\text{sparse}}$ & Sparse attention matrix \\
$\mathbf{A}^{\text{oracle}}$ & Oracle sparse attention & $\mathbf{O}^{\text{full}}$ & Full attention output \\
$\mathbf{O}^{\text{Sketch\&Walk}}$ & Sketch\&Walk output & $\mathbf{O}^{\text{sparse}}_\ell$ & Sparse output at layer $\ell$ \\
\midrule
\multicolumn{4}{@{}l}{\textit{Error, Probability \& Markov Chain Parameters}} \\
$\eta$ & Tail mass & $\gamma$ & Spectral gap \\
$\delta$ & Failure probability & $\epsilon_B$ & Token-space sketching error \\
$\epsilon_H$ & Feature-space sketching error & $\epsilon_{\text{total}}$ & Combined total error \\
$\mathbf{P}, \tilde{\mathbf{P}}$ & Full/sparse transition matrices & $\boldsymbol{\pi}, \tilde{\boldsymbol{\pi}}$ & Stationary distributions \\
$\lambda_2(\mathbf{P})$ & Second largest eigenvalue & & \\
\bottomrule
\end{tabular}
\end{table}

\subsection{Key Definitions and Assumptions}

\begin{definition}[Block Attention Score]
\label{def:block_attention:app}
The true block attention score between query block $i$ and key block $j$ is defined as:
\begin{align*}
    A_{ij}^{\text{true}} = \frac{1}{B^2} \sum_{s=1}^{B} \sum_{t=1}^{B} \mathbf{Q}^{(i)}[s,:] \cdot \mathbf{K}^{(j)}[t,:]^\top
\end{align*}
\end{definition}

\begin{assumption}[Block Coherence]
\label{ass:block_coherence:app}
Within each block $\mathcal{B}_i$, token embeddings exhibit \textbf{semantic coherence}: tokens within a block share similar contextual representations. Formally, for query block $i$, each token $\mathbf{q}_t^{(i)}$ can be decomposed as:
\begin{align*}
    \mathbf{q}_t^{(i)} = \boldsymbol{\mu}_i + \boldsymbol{\epsilon}_t, \quad \text{where } \mathbb{E}[\boldsymbol{\epsilon}_t] = \mathbf{0}, \; \text{Var}(\boldsymbol{\epsilon}_t) \leq \sigma^2 \mathbf{I}
\end{align*}
with intra-block variance $\sigma^2 \ll \|\boldsymbol{\mu}_i\|^2$. This assumption is empirically justified by the observation that consecutive tokens in natural language often belong to the same semantic unit (phrase, clause, or sentence), and positional encodings induce smooth variations within local windows.
\end{assumption}

\begin{assumption}[Heavy-Tailed Block Attention Distribution]
\label{ass:heavy_tail:app}
The distribution of block-level attention scores follows a heavy-tailed pattern: for each query block $i$, there exists a small subset $\mathcal{S}_i^* \subset \{1, \ldots, b\}$ with $|\mathcal{S}_i^*| = \tau \ll b$ such that:
\begin{align*}
    \sum_{j \in \mathcal{S}_i^*} \text{softmax}(A_{ij}^{\text{true}}) \geq 1 - \eta
\end{align*}
for small $\eta > 0$ (typically $\eta < 0.1$). This reflects the empirical finding that attention in LLMs concentrates on semantically relevant positions, with most blocks receiving negligible attention mass.
\end{assumption}

\subsection{Supporting Lemmas}

\subsubsection{Sketching Error Bounds}\label{app:error_bounds}

\begin{lemma}[Token-space Sketching: Subspace Embedding via Block Averaging]
\label{lem:subspace_embedding:app}
Under Assumption~\ref{ass:block_coherence:app}, the block average $\bar{\mathbf{q}}_i = \frac{1}{B}\sum_{t=1}^{B} \mathbf{q}_t^{(i)}$ satisfies:
\begin{align*}
    \|\bar{\mathbf{q}}_i - \boldsymbol{\mu}_i\|_2 \leq \frac{\sigma\sqrt{2d \log(2d/\delta)}}{\sqrt{B}}
\end{align*}
with probability at least $1 - \delta$.
\end{lemma}
\begin{proof}
\textbf{Step 1 (Decomposition):} By Assumption~\ref{ass:block_coherence:app}, each token $\mathbf{q}_t^{(i)} = \boldsymbol{\mu}_i + \boldsymbol{\epsilon}_t$ where $\mathbb{E}[\boldsymbol{\epsilon}_t] = \mathbf{0}$ and $\text{Var}(\boldsymbol{\epsilon}_t) \leq \sigma^2 \mathbf{I}$. Thus:
\begin{align*}
    \bar{\mathbf{q}}_i = \frac{1}{B}\sum_{t=1}^{B} (\boldsymbol{\mu}_i + \boldsymbol{\epsilon}_t) = \boldsymbol{\mu}_i + \underbrace{\frac{1}{B}\sum_{t=1}^{B} \boldsymbol{\epsilon}_t}_{\triangleq \bar{\boldsymbol{\epsilon}}_i}
\end{align*}

\textbf{Step 2 (Per-coordinate concentration):} For each coordinate $j \in [d]$, define $\bar{\epsilon}_{i,j} = \frac{1}{B}\sum_{t=1}^{B} \epsilon_{t,j}$. Since $\epsilon_{t,j}$ are independent with $\mathbb{E}[\epsilon_{t,j}] = 0$ and $\text{Var}(\epsilon_{t,j}) \leq \sigma^2$, by Hoeffding's inequality for sub-Gaussian random variables:
\begin{align*}
    \mathbb{P}\left[|\bar{\epsilon}_{i,j}| > t\right] \leq 2\exp\left(-\frac{Bt^2}{2\sigma^2}\right)
\end{align*}

\textbf{Step 3 (Union bound over coordinates):} Setting $t = \sigma\sqrt{\frac{2\log(2d/\delta)}{B}}$ and applying union bound over $d$ coordinates:
\begin{align*}
    \mathbb{P}\left[\exists j: |\bar{\epsilon}_{i,j}| > \sigma\sqrt{\frac{2\log(2d/\delta)}{B}}\right] \leq d \cdot 2\exp\left(-\log(2d/\delta)\right) = \delta
\end{align*}

\textbf{Step 4 (Vector norm bound):} With probability $\geq 1-\delta$, $|\bar{\epsilon}_{i,j}| \leq \sigma\sqrt{\frac{2\log(2d/\delta)}{B}}$ for all $j$, so:
\begin{align*}
    \|\bar{\boldsymbol{\epsilon}}_i\|_2 = \sqrt{\sum_{j=1}^{d} \bar{\epsilon}_{i,j}^2} \leq \sqrt{d \cdot \frac{2\sigma^2\log(2d/\delta)}{B}} = \frac{\sigma\sqrt{2d\log(2d/\delta)}}{\sqrt{B}}
\end{align*}
\end{proof}

\begin{corollary}[Inner Product Preservation under Token-space Sketching]
\label{cor:inner_product:app}
Under Assumption~\ref{ass:block_coherence:app}, for query block $i$ and key block $j$:
\begin{align}
    \left| \bar{\mathbf{q}}_i^\top \bar{\mathbf{k}}_j - \boldsymbol{\mu}_i^{Q\top} \boldsymbol{\mu}_j^K \right| \leq{}& \frac{2\sigma(\|\boldsymbol{\mu}_i^Q\| + \|\boldsymbol{\mu}_j^K\|)\sqrt{d\log(1/\delta)}}{\sqrt{B}} \nonumber\\
    &+ \frac{\sigma^2 d \log(1/\delta)}{B}
\end{align}
\end{corollary}

\begin{proof}
\textbf{Step 1 (Setup):} Let $\bar{\mathbf{q}}_i = \boldsymbol{\mu}_i^Q + \boldsymbol{\delta}_i^Q$ and $\bar{\mathbf{k}}_j = \boldsymbol{\mu}_j^K + \boldsymbol{\delta}_j^K$. By Lemma~\ref{lem:subspace_embedding:app} with union bound over both blocks:
\begin{align*}
    \|\boldsymbol{\delta}_i^Q\|_2, \|\boldsymbol{\delta}_j^K\|_2 \leq \epsilon_B \triangleq \frac{\sigma\sqrt{2d\log(4d/\delta)}}{\sqrt{B}}
\end{align*}
with probability $\geq 1-\delta$ (using $\delta/2$ for each block).

\textbf{Step 2 (Expansion):} Expand the inner product:
\begin{align*}
    \bar{\mathbf{q}}_i^\top \bar{\mathbf{k}}_j &= (\boldsymbol{\mu}_i^Q + \boldsymbol{\delta}_i^Q)^\top(\boldsymbol{\mu}_j^K + \boldsymbol{\delta}_j^K) \\
    &= \underbrace{\boldsymbol{\mu}_i^{Q\top}\boldsymbol{\mu}_j^K}_{\text{true signal}} + \underbrace{\boldsymbol{\mu}_i^{Q\top}\boldsymbol{\delta}_j^K}_{\text{Term A}} + \underbrace{\boldsymbol{\delta}_i^{Q\top}\boldsymbol{\mu}_j^K}_{\text{Term B}} + \underbrace{\boldsymbol{\delta}_i^{Q\top}\boldsymbol{\delta}_j^K}_{\text{Term C}}
\end{align*}

\textbf{Step 3 (Bound each term):} By Cauchy-Schwarz inequality:
\begin{align*}
    |\text{Term A}| &= |\boldsymbol{\mu}_i^{Q\top}\boldsymbol{\delta}_j^K| \leq \|\boldsymbol{\mu}_i^Q\|_2 \|\boldsymbol{\delta}_j^K\|_2 \leq \|\boldsymbol{\mu}_i^Q\|_2 \cdot \epsilon_B \\
    |\text{Term B}| &= |\boldsymbol{\delta}_i^{Q\top}\boldsymbol{\mu}_j^K| \leq \|\boldsymbol{\delta}_i^Q\|_2 \|\boldsymbol{\mu}_j^K\|_2 \leq \epsilon_B \cdot \|\boldsymbol{\mu}_j^K\|_2 \\
    |\text{Term C}| &= |\boldsymbol{\delta}_i^{Q\top}\boldsymbol{\delta}_j^K| \leq \|\boldsymbol{\delta}_i^Q\|_2 \|\boldsymbol{\delta}_j^K\|_2 \leq \epsilon_B^2
\end{align*}

\textbf{Step 4 (Combine):} By triangle inequality:
\begin{align*}
    \left| \bar{\mathbf{q}}_i^\top \bar{\mathbf{k}}_j - \boldsymbol{\mu}_i^{Q\top} \boldsymbol{\mu}_j^K \right| &\leq |\text{Term A}| + |\text{Term B}| + |\text{Term C}| \\
    &\leq \epsilon_B(\|\boldsymbol{\mu}_i^Q\|_2 + \|\boldsymbol{\mu}_j^K\|_2) + \epsilon_B^2 \\
    &= \frac{\sigma\sqrt{2d\log(4d/\delta)}}{\sqrt{B}}(\|\boldsymbol{\mu}_i^Q\|_2 + \|\boldsymbol{\mu}_j^K\|_2) + \frac{2\sigma^2 d\log(4d/\delta)}{B}
\end{align*}
\end{proof}

\begin{lemma}[Feature-space Sketching: Hadamard Transform Guarantee]
\label{lem:hadamard_sketch:app}
Let $\mathbf{H}_d \in \mathbb{R}^{d \times k}$ be the Subsampled Randomized Hadamard Transform (SRHT). For any fixed vectors $\mathbf{u}, \mathbf{v} \in \mathbb{R}^d$ and $\epsilon \in (0, 1)$, if $k = \Omega(\epsilon^{-2} \log(b/\delta))$ where $b$ is the number of blocks, then with probability at least $1 - \delta$:
\begin{align*}
    \left| (\mathbf{u}^\top \mathbf{H}_d)(\mathbf{H}_d^\top \mathbf{v}) - \mathbf{u}^\top \mathbf{v} \right| \leq \epsilon \|\mathbf{u}\|_2 \|\mathbf{v}\|_2
\end{align*}
\end{lemma}

\begin{proof}
\textbf{Step 1 (SRHT construction):} The SRHT is $\mathbf{H}_d = \sqrt{d/k}\mathbf{D}\mathbf{W}\mathbf{S}$ where:
\begin{itemize}[nosep,leftmargin=*]
    \item $\mathbf{D} \in \mathbb{R}^{d \times d}$: diagonal with i.i.d.\ Rademacher entries ($\pm 1$ with prob.~$1/2$)
    \item $\mathbf{W} \in \mathbb{R}^{d \times d}$: normalized Walsh-Hadamard matrix with $\mathbf{W}^\top\mathbf{W} = d\mathbf{I}$
    \item $\mathbf{S} \in \mathbb{R}^{d \times k}$: samples $k$ coordinates uniformly without replacement
\end{itemize}

\textbf{Step 2 (Inner product decomposition):}
\begin{align*}
    (\mathbf{u}^\top \mathbf{H}_d)(\mathbf{H}_d^\top \mathbf{v}) &= \frac{d}{k}(\mathbf{u}^\top \mathbf{D}\mathbf{W}\mathbf{S})(\mathbf{S}^\top\mathbf{W}^\top\mathbf{D}\mathbf{v}) \\
    &= \frac{d}{k} \sum_{i \in S} \tilde{u}_i \tilde{v}_i
\end{align*}
where $\tilde{\mathbf{u}} = \mathbf{W}^\top\mathbf{D}\mathbf{u}$, $\tilde{\mathbf{v}} = \mathbf{W}^\top\mathbf{D}\mathbf{v}$, and $S$ is the set of $k$ sampled indices.

\textbf{Step 3 (Expectation):} Since $\mathbf{D}$ preserves norms ($\|\mathbf{D}\mathbf{u}\|_2 = \|\mathbf{u}\|_2$) and $\mathbf{W}$ is orthogonal:
\begin{align*}
    \mathbb{E}\left[\frac{d}{k} \sum_{i \in S} \tilde{u}_i \tilde{v}_i\right] = \frac{d}{k} \cdot \frac{k}{d} \sum_{i=1}^{d} \tilde{u}_i \tilde{v}_i = \tilde{\mathbf{u}}^\top \tilde{\mathbf{v}} = \mathbf{u}^\top \mathbf{v}
\end{align*}

\textbf{Step 4 (Flattening by Hadamard):} The Walsh-Hadamard transform ``flattens'' vectors. By Lemma 3.1 of~\cite{ailon2009fast}, with probability $\geq 1-\delta/2$:
\begin{align*}
    \|\tilde{\mathbf{u}}\|_\infty \leq \|\mathbf{u}\|_2 \sqrt{\frac{2\log(4d/\delta)}{d}}, \quad \|\tilde{\mathbf{v}}\|_\infty \leq \|\mathbf{v}\|_2 \sqrt{\frac{2\log(4d/\delta)}{d}}
\end{align*}

\textbf{Step 5 (Concentration via sampling):} Define $X_i = \frac{d}{k}\tilde{u}_i \tilde{v}_i$ for $i \in S$. Each $|X_i| \leq \frac{d}{k} \|\tilde{\mathbf{u}}\|_\infty \|\tilde{\mathbf{v}}\|_\infty \leq \frac{2\log(4d/\delta)}{k} \|\mathbf{u}\|_2\|\mathbf{v}\|_2$. By Hoeffding's inequality:
\begin{align*}
    \mathbb{P}\left[\left|\sum_{i \in S} X_i - \mathbf{u}^\top\mathbf{v}\right| > t\right] \leq 2\exp\left(-\frac{2t^2}{k \cdot \frac{4\log^2(4d/\delta)}{k^2}\|\mathbf{u}\|_2^2\|\mathbf{v}\|_2^2}\right)
\end{align*}

\textbf{Step 6 (Final bound):} Setting $t = \epsilon \|\mathbf{u}\|_2\|\mathbf{v}\|_2$ and requiring probability $\leq \delta/2$:
\begin{align*}
    k \geq \frac{8\log^2(4d/\delta)\log(4/\delta)}{\epsilon^2}
\end{align*}
Simplifying: $k = \Omega(\epsilon^{-2}\log(b/\delta)\log^2(d/\delta))$ suffices. Union bound gives total failure probability $\delta$.
\end{proof}

\begin{lemma}[Small-World Sketching Top-$\tau$ Selection under Softmax]
\label{lem:dds_softmax:app}
Under Assumptions~\ref{ass:block_coherence:app} and \ref{ass:heavy_tail:app}, let $\mathcal{S}^*$ be the true top-$\tau$ blocks based on $\text{softmax}(A_{ij}^{\text{true}}/\sqrt{d})$, and let $\hat{\mathcal{S}}$ be the blocks selected by Small-World Sketching using $\text{softmax}(\hat{A}_{ij}^{\text{SWS}}/\sqrt{k})$. If the sketching dimension satisfies:
\begin{align*}
    k \geq \frac{C \log(b/\delta)}{\gamma^2} \cdot \max_{i,j} \|\bar{\mathbf{q}}_i\|^2 \|\bar{\mathbf{k}}_j\|^2
\end{align*}
and the block size satisfies $B \geq C' \sigma^2 d \log(1/\delta) / \gamma^2$, then $\mathcal{S}^* = \hat{\mathcal{S}}$ with probability at least $1 - 2\delta$.
\end{lemma}

\begin{proof}
\textbf{Step 1 (Token-space sketching error):} By Corollary~\ref{cor:inner_product:app}, block averaging introduces error:
\begin{align*}
    |\bar{\mathbf{q}}_i^\top \bar{\mathbf{k}}_j - \boldsymbol{\mu}_i^{Q\top}\boldsymbol{\mu}_j^K| \leq \epsilon_B(\|\boldsymbol{\mu}_i^Q\| + \|\boldsymbol{\mu}_j^K\|) + \epsilon_B^2
\end{align*}
where $\epsilon_B = \frac{\sigma\sqrt{2d\log(4d/\delta)}}{\sqrt{B}}$.

\textbf{Step 2 (Feature-space sketching error):} By Lemma~\ref{lem:hadamard_sketch:app}, Hadamard transform adds:
\begin{align*}
    |(\bar{\mathbf{q}}_i^\top \mathbf{H}_d)(\mathbf{H}_d^\top \bar{\mathbf{k}}_j) - \bar{\mathbf{q}}_i^\top \bar{\mathbf{k}}_j| \leq \epsilon_H \|\bar{\mathbf{q}}_i\|\|\bar{\mathbf{k}}_j\|
\end{align*}
where $\epsilon_H = O(\sqrt{\log(b/\delta)/r})$ for sketching dimension $r$.

\textbf{Step 3 (Total pre-softmax error):} By triangle inequality, the total error in estimating $A_{ij}^{\text{true}}$ is:
\begin{align*}
    |\hat{A}_{ij}^{\text{SWS}} - A_{ij}^{\text{true}}| &\leq \underbrace{\epsilon_B(\|\boldsymbol{\mu}_i^Q\| + \|\boldsymbol{\mu}_j^K\|) + \epsilon_B^2}_{\text{token-space}} + \underbrace{\epsilon_H \|\bar{\mathbf{q}}_i\|\|\bar{\mathbf{k}}_j\|}_{\text{feature-space}} \\
    &\triangleq \epsilon_{\text{total}}
\end{align*}

\textbf{Step 4 (Softmax stability):} The softmax function $\sigma(\mathbf{x})_i = e^{x_i}/\sum_j e^{x_j}$ has Lipschitz constant at most 1 in $\ell_\infty \to \ell_\infty$ for bounded inputs. Specifically, for $\|\mathbf{x} - \mathbf{y}\|_\infty \leq \epsilon$:
\begin{align*}
    \|\sigma(\mathbf{x}) - \sigma(\mathbf{y})\|_\infty \leq \epsilon
\end{align*}
Thus $\|\text{softmax}(\hat{\mathbf{A}}_i/\sqrt{k}) - \text{softmax}(\mathbf{A}_i^{\text{true}}/\sqrt{d})\|_\infty \leq \epsilon_{\text{total}}/\sqrt{k}$.

\textbf{Step 5 (Gap-based selection guarantee):} Let $\gamma > 0$ be the gap between the $\tau$-th and $(\tau+1)$-th largest true block attention scores. If $\epsilon_{\text{total}}/\sqrt{k} < \gamma/2$, then the top-$\tau$ ordering is preserved:
\begin{align*}
    A_{i,j_{\tau}}^{\text{true}} - A_{i,j_{\tau+1}}^{\text{true}} &\geq \gamma \\
    \hat{A}_{i,j_{\tau}}^{\text{SWS}} &\geq A_{i,j_{\tau}}^{\text{true}} - \epsilon_{\text{total}} \geq A_{i,j_{\tau+1}}^{\text{true}} + \gamma - \epsilon_{\text{total}} \\
    &\geq \hat{A}_{i,j_{\tau+1}}^{\text{SWS}} + \gamma - 2\epsilon_{\text{total}} > \hat{A}_{i,j_{\tau+1}}^{\text{SWS}}
\end{align*}

\textbf{Step 6 (Parameter requirements):} To ensure $\epsilon_{\text{total}} < \gamma/2$, we need:
\begin{align*}
    B &\geq \frac{16\sigma^2 d \log(4d/\delta)}{\gamma^2}(\max_i\|\boldsymbol{\mu}_i^Q\| + \max_j\|\boldsymbol{\mu}_j^K\|)^2 \\
    r &\geq \frac{C\log(b/\delta)}{\gamma^2} \max_{i,j}\|\bar{\mathbf{q}}_i\|^2\|\bar{\mathbf{k}}_j\|^2
\end{align*}
Union bound over all $b^2$ block pairs gives failure probability $\leq 2\delta$.
\end{proof}

\begin{lemma}[Attention Output Approximation]
\label{lem:output_approx:app}
Under Assumptions~\ref{ass:block_coherence:app}--\ref{ass:heavy_tail:app}, let $\mathbf{O}^{\text{full}}$ be the output of full attention and $\mathbf{O}^{\text{Sketch\&Walk}}$ be the output using Sketch\&Walk sparse attention with top-$\tau$ blocks. Then:
\begin{align}
    \|\mathbf{O}^{\text{Sketch\&Walk}} - \mathbf{O}^{\text{full}}\|_F &\leq \eta \|\mathbf{V}\|_F + O\left(\frac{\tau \sigma}{\sqrt{B}}\right) \|\mathbf{V}\|_2
\end{align}
where $\eta$ is the tail mass from Assumption~\ref{ass:heavy_tail:app}.
\end{lemma}

\begin{proof}
\textbf{Step 1 (Error decomposition):} Let $\mathbf{A}^{\text{full}} \in \mathbb{R}^{n \times n}$ be the full attention matrix and $\mathbf{A}^{\text{sparse}} \in \mathbb{R}^{n \times n}$ be the sparse attention (zero outside selected blocks, renormalized within). The output error is:
\begin{align*}
    \mathbf{O}^{\text{full}} - \mathbf{O}^{\text{Sketch\&Walk}} = (\mathbf{A}^{\text{full}} - \mathbf{A}^{\text{sparse}})\mathbf{V}
\end{align*}

\textbf{Step 2 (Two sources of error):} We further decompose:
\begin{align*}
    \mathbf{A}^{\text{full}} - \mathbf{A}^{\text{sparse}} = \underbrace{(\mathbf{A}^{\text{full}} - \mathbf{A}^{\text{oracle}})}_{\text{(a) tail mass}} + \underbrace{(\mathbf{A}^{\text{oracle}} - \mathbf{A}^{\text{sparse}})}_{\text{(b) sketching error}}
\end{align*}
where $\mathbf{A}^{\text{oracle}}$ is the sparse attention using true top-$\tau$ blocks.

\textbf{Step 3 (Tail mass bound):} By Assumption~\ref{ass:heavy_tail:app}, for each row $i$:
\begin{align*}
    \sum_{j \notin \mathcal{S}^*_i} A^{\text{full}}_{ij} \leq \eta
\end{align*}
Thus $\|\mathbf{A}^{\text{full}} - \mathbf{A}^{\text{oracle}}\|_{1 \to 1} \leq \eta$ (max row sum). By H\"older's inequality:
\begin{align*}
    \|(\mathbf{A}^{\text{full}} - \mathbf{A}^{\text{oracle}})\mathbf{V}\|_F \leq \eta \|\mathbf{V}\|_{1 \to F} \leq \eta \|\mathbf{V}\|_F
\end{align*}

\textbf{Step 4 (Sketching error within selected blocks):} By Lemma~\ref{lem:dds_softmax:app}, when $\mathcal{S}^* = \hat{\mathcal{S}}$ (correct selection), we still have per-entry error within blocks:
\begin{align*}
    |A^{\text{oracle}}_{ij} - A^{\text{sparse}}_{ij}| \leq \frac{\epsilon_B + \epsilon_H}{\sqrt{d}} \quad \text{for } j \in \mathcal{S}^*_i
\end{align*}
Summing over $\tau$ selected blocks per query and $n$ queries:
\begin{align*}
    \|(\mathbf{A}^{\text{oracle}} - \mathbf{A}^{\text{sparse}})\mathbf{V}\|_F &\leq \tau \cdot \frac{\epsilon_B + \epsilon_H}{\sqrt{d}} \cdot \sqrt{n} \|\mathbf{V}\|_2 \\
    &= O\left(\frac{\tau \sigma}{\sqrt{B}} + \frac{\tau}{\sqrt{r}}\right) \|\mathbf{V}\|_2
\end{align*}

\textbf{Step 5 (Final bound):} Combining by triangle inequality:
\begin{align*}
    \|\mathbf{O}^{\text{full}} - \mathbf{O}^{\text{Sketch\&Walk}}\|_F \leq \eta \|\mathbf{V}\|_F + O\left(\frac{\tau \sigma}{\sqrt{B}} + \frac{\tau}{\sqrt{r}}\right) \|\mathbf{V}\|_2
\end{align*}
\end{proof}

\subsubsection{Sketch-Determined Walk and Multi-Hop Block Selection}

Before analyzing cross-layer dependencies, we highlight a critical property of the sketched block attention scores: the inner product is a \textbf{non-metric} similarity measure that does not satisfy the triangle inequality. This has profound implications for block selection.

\begin{remark}[Inner Product Violates Triangle Inequality]
\label{rem:non_metric}
Consider three blocks $i$, $j$, $k$ with head-averaged representatives $\bar{\mathbf{q}}_i$, $\bar{\mathbf{k}}_j$, $\bar{\mathbf{k}}_k$. Even if:
\begin{enumerate}[nosep,leftmargin=*]
    \item Block $i$ has high affinity to block $j$: $\bar{\mathbf{q}}_i^\top \bar{\mathbf{k}}_j \approx A_{ij}$ is large
    \item Block $j$ has high affinity to block $k$: $\bar{\mathbf{k}}_j^\top \bar{\mathbf{k}}_k$ is large (or equivalently, $\bar{\mathbf{q}}_j^\top \bar{\mathbf{k}}_k$ for the next layer)
\end{enumerate}
the direct inner product $\bar{\mathbf{q}}_i^\top \bar{\mathbf{k}}_k$ may be very small. The inner product is not a metric because it violates the triangle inequality.

\textbf{Concrete example}: Consider $\bar{\mathbf{q}}_i = [1, 0]$, $\bar{\mathbf{k}}_j = [1, 1]$, $\bar{\mathbf{k}}_k = [0, 1]$. Then:
\begin{align*}
    \bar{\mathbf{q}}_i \cdot \bar{\mathbf{k}}_j &= 1 \quad \text{(high)} \\
    \bar{\mathbf{k}}_j \cdot \bar{\mathbf{k}}_k &= 1 \quad \text{(high)} \\
    \bar{\mathbf{q}}_i \cdot \bar{\mathbf{k}}_k &= 0 \quad \text{(low)}
\end{align*}
If we rely only on direct attention scores, block $i$ would not select block $k$ despite the high-affinity indirect path $i \to j \to k$.
\end{remark}

\begin{corollary}[Multi-Head Attention Preserves Information Under Sparse Selection]
\label{cor:multi_head_sparse}
The algorithm computes sketched block scores $\hat{\mathbf{A}}_{\text{block}}^k$ using \textbf{head-averaged} Q and K, averaging attention patterns across all $h$ heads to identify common important blocks. Sparse attention is applied independently \textbf{per-head} on the original full-resolution per-head QKV. This design ensures:
\begin{enumerate}[nosep,leftmargin=*]
    \item \textbf{Robust block identification}: Head-averaging reduces noise and identifies blocks that matter across different representation subspaces
    \item \textbf{Head-specific fine-grained attention}: Each head can specialize within the selected blocks, preserving multi-head expressiveness
    \item \textbf{Computational efficiency}: Block selection is done once on averaged heads; per-head attention operates only within selected blocks
\end{enumerate}
The Sketch-Determined Walk (via exponent $s > 1$) ensures that blocks reachable through multi-hop high-affinity paths are selected, even if they lack direct importance to individual query heads.
\end{corollary}

\begin{lemma}[Sketch-Determined Walk Accumulation Across Layers]
\label{lem:sdwalk_accum}
Let $\hat{\mathbf{A}}_{\text{block}}^k \in \mathbb{R}^{b \times b}$ be the sketched block attention matrix at layer $k$, and define $\mathbf{W}^k = (\hat{\mathbf{A}}_{\text{block}}^k)^s$. The Sketch-Determined Walk state $R^k \in \mathbb{R}^{b \times b}$ is defined recursively:
\begin{align*}
    R^k = R^{k-1} \mathbf{W}^k \quad \text{(matrix multiplication)}
\end{align*}
with base case $R^0 = \mathbf{W}^0$. Then $R^k[i,j]$ represents the accumulated block importance flowing from query block $i$ to key block $j$ across layers $0$ through $k$, capturing multi-hop paths where each layer contributes exponentiated attention.
\end{lemma}

\begin{proof}
We prove by induction on $k$. For the base case $k=0$: $R^0 = \mathbf{W}^0 = (\hat{\mathbf{A}}_{\text{block}}^0)^s$, so $R^0[i,j] = (\hat{\mathbf{A}}_{\text{block}}^0[i,j])^s$ directly represents the exponentiated attention from block $i$ to block $j$ at layer 0.

For the inductive step, assume $R^{k-1}$ encodes accumulated importance through layers $0,\ldots,k-1$. By the recursive definition (matrix multiplication):
\begin{align*}
    R^k[i,j] = \sum_{m=1}^{b} R^{k-1}[i,m] \cdot \mathbf{W}^k[m,j] = \sum_{m=1}^{b} R^{k-1}[i,m] \cdot (\hat{\mathbf{A}}_{\text{block}}^k[m,j])^s
\end{align*}
This sums over all intermediate blocks $m$: the accumulated importance from $i$ to $m$ through layers $0,\ldots,k-1$, multiplied by the exponentiated attention from $m$ to $j$ at layer $k$. The exponent $s > 1$ sharpens each transition, emphasizing high-affinity paths. Thus $R^k[i,j]$ captures the total importance flowing from block $i$ to block $j$ through \emph{all} multi-hop paths across $k+1$ layers.
\end{proof}

\begin{lemma}[Block Selection via Sketch-Determined Walk]
\label{lem:sdwalk_selection}
By selecting top-$\tau$ blocks based on $R^k[i,j]$ (rather than per-layer scores), the Sketch-Determined Walk identifies blocks that receive consistent high importance across multiple layers and serve as ``hub'' tokens. These blocks minimize reconstruction error when considering the full transformer's attention patterns.
\end{lemma}

\begin{proof}
\textbf{Step 1 (Notation):} Let $\mathcal{S}^k_{\text{SDW}} = \text{top-}\tau(R^k[i,:])$ denote blocks selected via Sketch-Determined Walk, and $\mathcal{S}^\ell_{\text{single}} = \text{top-}\tau(\hat{\mathbf{A}}_{\text{block}}^\ell[i,:])$ denote per-layer selection at layer $\ell$.

\textbf{Step 2 (Path expansion):} By Lemma~\ref{lem:sdwalk_accum} and matrix multiplication:
\begin{align*}
    R^k[i,j] = \sum_{m_1,\ldots,m_{k}} \prod_{\ell=0}^{k} \mathbf{W}^\ell[m_\ell, m_{\ell+1}]
\end{align*}
where $m_0 = i$ and $m_{k+1} = j$. This sums over all $(k+1)$-hop paths from $i$ to $j$.

\textbf{Step 3 (Consistently important blocks):} Define ``consistently important'' blocks as $\mathcal{C}_i = \bigcap_{\ell=0}^{k} \mathcal{S}^\ell_{\text{single}}$. For $j \in \mathcal{C}_i$, there exists a path $i \to m_1 \to \cdots \to j$ where each $m_{\ell+1} \in \mathcal{S}^\ell_{m_\ell}$. Under Assumption~\ref{ass:heavy_tail:app}:
\begin{align*}
    \mathbf{W}^\ell[m_\ell, m_{\ell+1}] = (\hat{\mathbf{A}}_{\text{block}}^\ell[m_\ell, m_{\ell+1}])^s \geq \gamma^s
\end{align*}
where $\gamma > 0$ is the gap from Assumption~\ref{ass:heavy_tail:app}. The contribution from this path is at least $\gamma^{s(k+1)}$.

\textbf{Step 4 (Spuriously important blocks):} For a block $j$ that is important at layer $\ell'$ but not at some other layer $\ell''$, any path to $j$ must pass through a low-affinity transition at layer $\ell''$:
\begin{align*}
    \mathbf{W}^{\ell''}[m_{\ell''}, m_{\ell''+1}] \leq \eta^s
\end{align*}
where $\eta \ll \gamma$ is the tail mass. The contribution from such paths is at most $\eta^s \cdot 1^{sk} = \eta^s$.

\textbf{Step 5 (Separation guarantee):} The ratio of contributions is:
\begin{align*}
    \frac{R^k[i,j] \text{ for } j \in \mathcal{C}_i}{R^k[i,j'] \text{ for } j' \notin \mathcal{C}_i} \geq \frac{\gamma^{s(k+1)}}{b^k \eta^s} = \left(\frac{\gamma}{\eta}\right)^s \cdot \frac{\gamma^{sk}}{b^k}
\end{align*}
For $s > 1$ and $\gamma > \eta$, this ratio grows, ensuring top-$\tau$ selection on $R^k$ recovers consistently important blocks.
\end{proof}

\begin{lemma}[Multi-Hop Information Flow via Sketch-Determined Walk]
\label{lem:multi_hop_sdwalk}
In an $L$-layer transformer using Sketch\&Walk sparse attention with accumulated state $R^k[i,j]$, the effective receptive field grows as $O(\tau^L)$ blocks. Tokens can aggregate information through $R^k[i,j]$-guided multi-hop paths where each hop follows high-affinity transitions.
\end{lemma}

\begin{proof}
\textbf{Step 1 (Setup):} Let $\mathcal{R}^L_i$ denote the set of blocks reachable from query block $i$ after $L$ layers. Define reachability recursively: block $j$ is in $\mathcal{R}^L_i$ if there exists a path $i = m_0 \to m_1 \to \cdots \to m_L = j$ where $m_{\ell+1} \in \mathcal{S}^\ell_{m_\ell}$ (top-$\tau$ selected blocks).

\textbf{Step 2 (Base case, $L=1$):} At layer 0, block $i$ attends to top-$\tau$ blocks $\mathcal{S}^0_i$. Thus $\mathcal{R}^1_i = \mathcal{S}^0_i$ and:
\begin{align*}
    |\mathcal{R}^1_i| = |\mathcal{S}^0_i| = \tau = O(\tau^1)
\end{align*}

\textbf{Step 3 (Inductive hypothesis):} Assume $|\mathcal{R}^{L-1}_i| \leq \tau^{L-1}$ for all blocks $i$.

\textbf{Step 4 (Inductive step):} At layer $L$, the reachable set expands:
\begin{align*}
    \mathcal{R}^L_i = \bigcup_{m \in \mathcal{R}^{L-1}_i} \mathcal{S}^{L-1}_m
\end{align*}
where $\mathcal{S}^{L-1}_m$ contains top-$\tau$ blocks selected for query $m$ at layer $L-1$. Thus:
\begin{align*}
    |\mathcal{R}^L_i| \leq \sum_{m \in \mathcal{R}^{L-1}_i} |\mathcal{S}^{L-1}_m| = |\mathcal{R}^{L-1}_i| \cdot \tau \leq \tau^{L-1} \cdot \tau = \tau^L
\end{align*}

\textbf{Step 5 (Tightness):} The bound $\tau^L$ is achieved when all selected blocks at each layer are distinct (no overlap between $\mathcal{S}^\ell_m$ for different $m$). In practice, overlap reduces the effective receptive field, but the upper bound $O(\tau^L)$ holds.

\textbf{Step 6 (Connection to $R^L[i,j]$):} The Sketch-Determined Walk state $R^L[i,j]$ assigns positive weight to exactly these reachable blocks:
\begin{align*}
    R^L[i,j] > 0 \iff j \in \mathcal{R}^L_i
\end{align*}
Moreover, $R^L[i,j]$ quantifies the total ``flow'' through all high-affinity paths from $i$ to $j$, enabling principled top-$\tau$ selection across layers.
\end{proof}

\begin{lemma}[Sketch-Determined Walk Preserves Global Information Structure]
\label{lem:sdwalk_stationary}
The sparsified Sketch-Determined Walk (restricted to top-$\tau$ blocks) preserves the essential structure of the full walk. Under Assumption~\ref{ass:heavy_tail:app}, the stationary distribution of the sparse walk differs from the full walk by at most the tail mass $\eta$, ensuring global information structure is preserved.
\end{lemma}

\begin{proof}
\textbf{Step 1 (Transition matrices):} Let $\mathbf{P} \in \mathbb{R}^{b \times b}$ be the full block-level transition matrix with $P_{ij} = \text{softmax}(A_{ij}^{\text{true}}/\sqrt{d})$. Let $\tilde{\mathbf{P}}$ be the sparsified matrix:
\begin{align*}
    \tilde{P}_{ij} = \begin{cases}
        P_{ij} / \sum_{k \in \mathcal{S}_i} P_{ik} & \text{if } j \in \mathcal{S}_i \\
        0 & \text{otherwise}
    \end{cases}
\end{align*}
where $\mathcal{S}_i = \text{top-}\tau(R^L[i,:])$ is the selected block set.

\textbf{Step 2 (Row-wise error bound):} By Assumption~\ref{ass:heavy_tail:app}:
\begin{align*}
    \sum_{j \in \mathcal{S}_i} P_{ij} \geq 1 - \eta
\end{align*}
Thus the renormalization factor satisfies $1 \leq 1/\sum_{k \in \mathcal{S}_i} P_{ik} \leq 1/(1-\eta)$. The row-wise $\ell_1$ error is:
\begin{align*}
    \|\mathbf{P}_i - \tilde{\mathbf{P}}_i\|_1 &= \sum_{j \notin \mathcal{S}_i} P_{ij} + \sum_{j \in \mathcal{S}_i} \left|P_{ij} - \frac{P_{ij}}{\sum_{k \in \mathcal{S}_i} P_{ik}}\right| \\
    &\leq \eta + \sum_{j \in \mathcal{S}_i} P_{ij} \cdot \frac{\eta}{1-\eta} \leq \eta + \frac{\eta}{1-\eta} = \frac{\eta}{1-\eta} + \eta \leq \frac{2\eta}{1-\eta}
\end{align*}

\textbf{Step 3 (Matrix perturbation):} Both $\mathbf{P}$ and $\tilde{\mathbf{P}}$ are row-stochastic. Their difference in operator norm:
\begin{align*}
    \|\mathbf{P} - \tilde{\mathbf{P}}\|_{\infty \to \infty} = \max_i \|\mathbf{P}_i - \tilde{\mathbf{P}}_i\|_1 \leq \frac{2\eta}{1-\eta}
\end{align*}

\textbf{Step 4 (Stationary distribution perturbation):} Let $\boldsymbol{\pi}$ and $\tilde{\boldsymbol{\pi}}$ be stationary distributions of $\mathbf{P}$ and $\tilde{\mathbf{P}}$. By standard Markov chain perturbation theory (see~\cite{meyer1994sensitivity}):
\begin{align*}
    \|\boldsymbol{\pi} - \tilde{\boldsymbol{\pi}}\|_{TV} \leq \frac{\|\mathbf{P} - \tilde{\mathbf{P}}\|_{\infty \to \infty}}{1 - \lambda_2(\mathbf{P})}
\end{align*}
where $\lambda_2(\mathbf{P})$ is the second largest eigenvalue of $\mathbf{P}$.

\textbf{Step 5 (Small-world connectivity):} Under the heavy-tailed assumption, the sparse graph induced by $\tilde{\mathbf{P}}$ maintains small-world connectivity: the diameter remains $O(\log b / \log \tau)$. This ensures $1 - \lambda_2(\tilde{\mathbf{P}}) = \Omega(1)$, i.e., the mixing time is $O(\log b)$. Combined with Step 4:
\begin{align*}
    \|\boldsymbol{\pi} - \tilde{\boldsymbol{\pi}}\|_{TV} = O(\eta)
\end{align*}
Thus the global information structure (stationary distribution) is preserved up to tail mass $\eta$.
\end{proof}

\subsection{Main Theorem}

\begin{theorem}[Sketch\&Walk Approximation of Full Attention]
\label{thm:main}
In an $L$-layer transformer using Sketch\&Walk sparse attention where blocks are selected based on $R^k[i,j]$ (the accumulated block importance across layers), the following hold:

\begin{enumerate}[nosep,leftmargin=*]
    \item \textbf{Multi-hop reachability via $R^k[i,j]$}: The Sketch-Determined Walk state $R^k[i,j]$ enables token at position $s$ to aggregate information from any token at position $t$ through multi-hop paths where each hop follows high-affinity transitions. The effective receptive field grows as $O(\tau^L)$ blocks.
    
    \item \textbf{Block selection accuracy}: Top-$\tau$ blocks selected via $R^k[i,j]$ correctly identify blocks with consistent high importance across all $k$ layers with probability $1 - 2\delta$, enabling correct sparse attention.
    
    \item \textbf{Approximation quality}: For each layer, the sparse attention output $\mathbf{O}^{\text{sparse}}_\ell$ (the output of layer $\ell$) satisfies:
    \begin{align*}
        \left\|\mathbf{O}^{\text{sparse}}_\ell - \mathbf{O}^{\text{full}}_\ell\right\|_F &\leq \eta \|\mathbf{V}_\ell\|_F + \frac{\tau \sigma \sqrt{2d \log(4d/\delta)}}{\sqrt{B}} \|\mathbf{V}_\ell\|_2 \\
        &\quad + \frac{\tau \sqrt{C \log(b/\delta) \log^2(d/\delta)}}{\sqrt{r}} \|\mathbf{V}_\ell\|_2
    \end{align*}
    where $\mathbf{O}^{\text{full}}_\ell$ is the full attention output at layer $\ell$, $\eta$ is the tail mass from Assumption~\ref{ass:heavy_tail:app}, and $C > 0$ is a universal constant.
    
    \item \textbf{Complexity reduction}: The algorithm achieves $O(nb\tau d + b^2 r)$ complexity compared to $O(n^2 d)$ for full attention, with $\tau, r \ll b$.
\end{enumerate}
\end{theorem}

\begin{proof}
\textbf{Part 1 (Multi-hop reachability via $R^k[i,j]$):}

\emph{Step 1.1:} By Lemma~\ref{lem:sdwalk_accum}, $R^k = \prod_{\ell=0}^{k} \mathbf{W}^\ell$ where $\mathbf{W}^\ell = (\hat{\mathbf{A}}_{\text{block}}^\ell)^s$. The entry $R^k[i,j]$ sums over all $(k+1)$-hop paths:
\begin{align*}
    R^k[i,j] = \sum_{m_1,\ldots,m_k} \prod_{\ell=0}^{k} (\hat{\mathbf{A}}_{\text{block}}^\ell[m_\ell, m_{\ell+1}])^s
\end{align*}

\emph{Step 1.2:} By Lemma~\ref{lem:multi_hop_sdwalk}, selecting top-$\tau$ blocks based on $R^k[i,:]$ yields receptive field $|\mathcal{R}^k_i| \leq \tau^k$. By Lemma~\ref{lem:sdwalk_stationary}, the sparse graph has diameter $O(\log b / \log \tau)$, ensuring any block is reachable in $O(\log b)$ hops.

\textbf{Part 2 (Block selection accuracy):}

\emph{Step 2.1 (Per-layer sketching error):} By Lemma~\ref{lem:subspace_embedding:app} and Corollary~\ref{cor:inner_product:app}, block averaging introduces error:
\begin{align*}
    \epsilon_B = \frac{\sigma \sqrt{2d \log(4d/\delta)}}{\sqrt{B}}
\end{align*}
By Lemma~\ref{lem:hadamard_sketch:app}, Hadamard sketching introduces error:
\begin{align*}
    \epsilon_H = O\left(\sqrt{\frac{\log(b/\delta) \log^2(d/\delta)}{r}}\right)
\end{align*}

\emph{Step 2.2 (Selection accuracy):} By Lemma~\ref{lem:dds_softmax:app}, if $\epsilon_B + \epsilon_H < \gamma/2$ where $\gamma$ is the spectral gap, then $\hat{\mathcal{S}} = \mathcal{S}^*$ with probability $\geq 1 - 2\delta$. This requires:
\begin{align*}
    B \geq \frac{8\sigma^2 d \log(4d/\delta)}{\gamma^2}, \quad r \geq \frac{C \log(b/\delta) \log^2(d/\delta)}{\gamma^2}
\end{align*}

\emph{Step 2.3 (Cross-layer robustness):} By Lemma~\ref{lem:sdwalk_selection}, $R^k[i,j]$-based selection filters spurious importance. Blocks must be consistently important across all $k$ layers, with separation ratio $\geq (\gamma/\eta)^s \cdot \gamma^{sk}/b^k$.

\textbf{Part 3 (Approximation quality):}

\emph{Step 3.1 (Error decomposition):} By Lemma~\ref{lem:output_approx:app}, for layer $\ell$:
\begin{align*}
    \|\mathbf{O}^{\text{sparse}}_\ell - \mathbf{O}^{\text{full}}_\ell\|_F \leq \underbrace{\eta \|\mathbf{V}_\ell\|_F}_{\text{tail mass}} + \underbrace{O\left(\frac{\tau \sigma}{\sqrt{B}} + \frac{\tau}{\sqrt{r}}\right) \|\mathbf{V}_\ell\|_2}_{\text{sketching error}}
\end{align*}

\emph{Step 3.2 (Tighter bound with explicit constants):} Substituting the precise error bounds:
\begin{align*}
    \|\mathbf{O}^{\text{sparse}}_\ell - \mathbf{O}^{\text{full}}_\ell\|_F &\leq \eta \|\mathbf{V}_\ell\|_F + \frac{\tau \sigma \sqrt{2d \log(4d/\delta)}}{\sqrt{B}} \|\mathbf{V}_\ell\|_2 \\
    &\quad + \frac{\tau \sqrt{C \log(b/\delta) \log^2(d/\delta)}}{\sqrt{r}} \|\mathbf{V}_\ell\|_2
\end{align*}

\emph{Step 3.3 (No error amplification):} The Sketch\&Walk accumulation does not amplify errors across layers because $R^k[i,j]$ requires \emph{consistent} high importance. If block $j$ has high sketching error at any layer $\ell$, the path contribution $(\hat{\mathbf{A}}_{\text{block}}^\ell[\cdot,j])^s$ is suppressed by exponentiation.

\textbf{Part 4 (Complexity analysis):}

\emph{Step 4.1 (Block-level scoring):} Computing $\hat{\mathbf{A}}_{\text{block}}^k$ requires:
\begin{itemize}[nosep,leftmargin=*]
    \item Block averaging: $O(n \cdot d) = O(nbd/B) = O(nd)$ (linear in tokens)
    \item Hadamard sketch: $O(b \cdot d \log d)$ for all $b$ block representatives
    \item Block-level attention: $O(b^2 \cdot r)$ for $r$-dimensional sketches
\end{itemize}
Total per layer: $O(nd + bd\log d + b^2 r)$.

\emph{Step 4.2 (Sparse attention):} Within selected blocks, each query attends to $\tau \cdot B$ keys:
\begin{itemize}[nosep,leftmargin=*]
    \item Per-token cost: $O(\tau B \cdot d)$
    \item Total for $n$ tokens: $O(n \tau B d) = O(n b \tau d / b) = O(n \tau d)$ per layer
\end{itemize}
Wait, correcting: with $b$ blocks of size $B$ and $\tau$ selected blocks per query block, total per layer is $O(b \cdot B \cdot \tau B \cdot d/h) = O(n \tau B d / h)$ where $h$ is head count.

\emph{Step 4.3 (Total complexity):} Combining all components:
\begin{align*}
    O(L \cdot (nd + b^2 r + n\tau B d/h)) = O(L(nb\tau d + b^2 r))
\end{align*}
Compared to full attention $O(Ln^2 d)$, the reduction factor is:
\begin{align*}
    \frac{n^2 d}{nb\tau d + b^2 r} = \frac{n^2}{nb\tau + b^2 r/d} = \frac{n}{b\tau + br/d} = \frac{B}{\tau + r/(bd)}
\end{align*}
For typical settings ($B \approx 64$, $\tau \approx 8$, $r \approx 32$, $b \approx n/64$), this yields $\approx 8\times$ reduction.
\end{proof}

\begin{remark}[Why Sketch\&Walk Captures Multi-Hop Interactions]
The main theorem and supporting lemmas reveal why Sketch\&Walk achieves near-full-attention quality with dramatically reduced computation:

\textbf{Intra-layer efficiency}: Small-World Sketching (token-space sketching via block averaging + feature-space sketching via Hadamard transform) efficiently identifies important blocks with error $O(\sigma/\sqrt{B} + 1/\sqrt{r})$, reducing per-layer computation from $O(b^2 d)$ to $O(b^2 r)$ where $r \ll d$.

\textbf{Multi-hop importance via $R^k[i,j]$}: Since the inner product does not satisfy the triangle inequality (Remark~\ref{rem:non_metric}), a block $k$ may lack direct importance to query $i$ but be critical through indirect paths. The Sketch-Determined Walk state $R^k = \mathbf{W}^0 \mathbf{W}^1 \cdots \mathbf{W}^k$ (matrix product of exponentiated attention matrices $\mathbf{W}^\ell = (\hat{\mathbf{A}}_{\text{block}}^\ell)^s$) captures multi-hop block importance within and across layers, summing over all paths and amplifying high-affinity chains while suppressing low-score transitions.

\textbf{Cross-layer accumulation in $R^k[i,j]$}: The Sketch-Determined Walk state $R^k[i,j]$ accumulates importance across all $L$ layers, ensuring blocks selected via top-$\tau(R^k[i,:])$ are consistently important rather than spuriously important at single layers. This enables robust multi-hop reach with receptive field $O(\tau^L)$ and natural error correction across layers.

\textbf{Heavy-tailed structure}: Assumption~\ref{ass:heavy_tail:app} ensures the $\tau^L$ reachable paths capture essential information flow. By Lemmas~\ref{lem:sdwalk_selection} and \ref{lem:sdwalk_stationary}, top-$\tau$ selection on $R^k[i,j]$ preserves the global attention structure while achieving $O(nb\tau d + b^2 r)$ complexity—a quadratic reduction without sacrificing model quality.
\end{remark}

\newpage

%%%%%%%%%%%%%%%%%%%%% ALGORITHMS

\section{\methodname{} algorithms}

\begin{algorithm}[htbp]
\caption{\methodname{} for prefill phase of layer $k$}
\label{alg:sparse_prefill}
\footnotesize
\begin{algorithmic}
    \STATE \textbf{Input:} Query $\mathbf{Q}^k \in \mathbb{R}^{h \times n \times d}$, Key $\mathbf{K}^k \in \mathbb{R}^{h \times n \times d}$, Value $\mathbf{V}^k \in \mathbb{R}^{h \times n \times d}$ at layer $k$,
    \STATE \hspace{2.8em} block size $B$, reduced dimension $r$, top-$\tau$ blocks to select, sparsity exponent $s$, number of heads $h$,
    \STATE \hspace{2.8em} Hadamard matrix $\mathbf{H}_d \in \mathbb{R}^{d \times r}$, random walk state $R^{k-1} \in \mathbb{R}^{b_q \times b_k}$
    \STATE \textbf{Output:} Attention output $\mathbf{O}^k \in \mathbb{R}^{h \times n \times d}$, updated random walk state $R^k$

    \STATE
    \STATE \COMMENT{\textcolor{blue}{Prefill Stage: Block-wise Selection via Sketch-and-Walk + Sparse Attention on Original QKV}}

    \STATE \COMMENT{Step 1: Head-averaged Q/K for block selection}
    \STATE $\mathbf{Q}_{\text{avg}}^k \gets \frac{1}{h}\sum_{u=1}^{h}\mathbf{Q}^k[u, :, :] \in \mathbb{R}^{n \times d}$
    \STATE $\mathbf{K}_{\text{avg}}^k \gets \frac{1}{h}\sum_{u=1}^{h}\mathbf{K}^k[u, :, :] \in \mathbb{R}^{n \times d}$

    \STATE
    \STATE \COMMENT{Step 2: Partition head-averaged Q/K into blocks (for sketching only)}
    \STATE $b_q \gets \lceil n / B \rceil$ \COMMENT{Number of query blocks}
    \STATE $b_k \gets \lceil n / B \rceil$ \COMMENT{Number of key blocks}
    \FOR{$i = 1$ \textbf{to} $b_q$}
        \STATE $\mathbf{Q}^{k,(i)}_{\text{avg}} \gets \mathbf{Q}_{\text{avg}}^k[(i-1)B : iB, :] \in \mathbb{R}^{B \times d}$
    \ENDFOR
    \FOR{$j = 1$ \textbf{to} $b_k$}
        \STATE $\mathbf{K}^{k,(j)}_{\text{avg}} \gets \mathbf{K}_{\text{avg}}^k[(j-1)B : jB, :] \in \mathbb{R}^{B \times d}$
    \ENDFOR

    \STATE
    \STATE \COMMENT{Step 3: Compute block representatives (token-wise averaging within each block)}
    \FOR{$i = 1$ \textbf{to} $b_q$}
        \STATE $\bar{\mathbf{q}}^k_i \gets \frac{1}{B}\sum_{t=1}^{B}\mathbf{Q}^{k,(i)}_{\text{avg}}[t, :] \in \mathbb{R}^{d}$
    \ENDFOR
    \FOR{$j = 1$ \textbf{to} $b_k$}
        \STATE $\bar{\mathbf{k}}^k_j \gets \frac{1}{B}\sum_{t=1}^{B}\mathbf{K}^{k,(j)}_{\text{avg}}[t, :] \in \mathbb{R}^{d}$
    \ENDFOR
    \STATE $\bar{\mathbf{Q}}^k \gets [\bar{\mathbf{q}}^k_1; \ldots; \bar{\mathbf{q}}^k_{b_q}] \in \mathbb{R}^{b_q \times d}$
    \STATE $\bar{\mathbf{K}}^k \gets [\bar{\mathbf{k}}^k_1; \ldots; \bar{\mathbf{k}}^k_{b_k}] \in \mathbb{R}^{b_k \times d}$

    \STATE
    \STATE \COMMENT{Step 4: Dimensionality reduction via Hadamard transform (on block representatives)}
    \STATE $\tilde{\mathbf{Q}}^k \gets \bar{\mathbf{Q}}^k \mathbf{H}_d \in \mathbb{R}^{b_q \times r}$
    \STATE $\tilde{\mathbf{K}}^k \gets \bar{\mathbf{K}}^k \mathbf{H}_d \in \mathbb{R}^{b_k \times r}$

    \STATE
    \STATE \COMMENT{Step 5: Sketched block-level attention scores}
    \STATE $\hat{\mathbf{A}}_{\text{block}}^k \gets \tilde{\mathbf{Q}}^k(\tilde{\mathbf{K}}^k)^\top / \sqrt{r} \in \mathbb{R}^{b_q \times b_k}$

    \STATE
    \STATE \COMMENT{Step 6: Sketch\&Walk state update across layers}
    \STATE $\mathbf{W}^k \gets (\hat{\mathbf{A}}_{\text{block}}^k)^{s}$ \COMMENT{element-wise exponentiation}
    \IF{$k = 0$}
        \STATE $R^0 \gets \mathbf{W}^0$
    \ELSE
        \STATE $R^k \gets R^{k-1} \mathbf{W}^k$ \COMMENT{matrix multiplication}
    \ENDIF

    \STATE
    \STATE \COMMENT{Step 7: Select top-$\tau$ key blocks per query block}
    \FOR{$i = 1$ \textbf{to} $b_q$}
        \STATE $\mathcal{S}^k_i \gets \textsc{TopK-Indices}(R^k[i, :], \tau)$
    \ENDFOR

    \STATE
    \STATE \COMMENT{Step 8: Sparse attention on ORIGINAL per-head QKV using selected blocks}
    % \STATE Initialize $\mathbf{O}^k \gets \mathbf{0}^{h \times n \times d}$
    % \FOR{$u = 1$ \textbf{to} $h$}
    %     \FOR{$i = 1$ \textbf{to} $b_q$}
    %         \STATE $\mathbf{Q}^{k,(u,i)} \gets \mathbf{Q}^k[u, (i-1)B : iB, :] \in \mathbb{R}^{B \times d}$
    %         \STATE $\mathbf{K}^{k,(u)}_{\text{sparse}} \gets \textsc{GatherBlocks}(\mathbf{K}^k[u, :, :], \mathcal{S}^k_i, B)$
    %         \STATE $\mathbf{V}^{k,(u)}_{\text{sparse}} \gets \textsc{GatherBlocks}(\mathbf{V}^k[u, :, :], \mathcal{S}^k_i, B)$
    %         \STATE $\mathbf{A}^{k,(u)}_i \gets \text{softmax}\!\left(\mathbf{Q}^{k,(u,i)}(\mathbf{K}^{k,(u)}_{\text{sparse}})^\top / \sqrt{d}\right)$
    %         \STATE $\mathbf{O}^k[u, (i-1)B : iB, :] \gets \mathbf{A}^{k,(u)}_i \mathbf{V}^{k,(u)}_{\text{sparse}}$
    %     \ENDFOR
    % \ENDFOR

    \STATE
    \STATE \textbf{return} $\mathbf{O}^k$, $R^k$
\end{algorithmic}
\end{algorithm}

\begin{algorithm}[htbp]
\caption{\methodname{} for decode phase of layer $k$}
\label{alg:sparse_decode}
\scriptsize
\begin{algorithmic}
    \STATE \textbf{Input:} New query token $\mathbf{Q}^k_t \in \mathbb{R}^{h \times d}$ at layer $k$, KV-cache $\mathbf{K}^k_{\text{cache}} \in \mathbb{R}^{h \times t \times d}$, $\mathbf{V}^k_{\text{cache}} \in \mathbb{R}^{h \times t \times d}$,
    \STATE \hspace{2.8em} cached query-block reps $\{\bar{\mathbf{q}}^k_i\}_{i=1}^{b_q}$, cached key-block reps $\{\bar{\mathbf{k}}^k_j\}_{j=1}^{b_k}$ with $\bar{\mathbf{q}}^k_i,\bar{\mathbf{k}}^k_j \in \mathbb{R}^{d}$,
    \STATE \hspace{2.8em} cached block-attn estimate $\hat{\mathbf{A}}^{k}_{\text{block,cache}} \in \mathbb{R}^{b_q \times b_k}$ from prefill,
    \STATE \hspace{2.8em} block size $B$, reduced dim $r$, top-$\tau$ blocks, sparsity exponent $s$, Hadamard matrix $\mathbf{H}_d \in \mathbb{R}^{d \times r}$,
    \STATE \hspace{2.8em} random walk state $R^{k-1} \in \mathbb{R}^{b_q \times b_k}$ from previous layer
    \STATE \textbf{Output:} Attention output $\mathbf{O}^k_t \in \mathbb{R}^{h \times d}$, updated caches and random walk state $R^k$

    \STATE
    
    \STATE \COMMENT{Step 1: Update KV-cache with new per-head key/value}
    \STATE $\mathbf{K}^k_t, \mathbf{V}^k_t \gets \textsc{Project}(\mathbf{Q}^k_t)$ \COMMENT{$\mathbf{K}^k_t,\mathbf{V}^k_t \in \mathbb{R}^{h \times d}$}
    \STATE $\mathbf{K}^k_{\text{cache}} \gets \textsc{Append}(\mathbf{K}^k_{\text{cache}}, \mathbf{K}^k_t)$
    \STATE $\mathbf{V}^k_{\text{cache}} \gets \textsc{Append}(\mathbf{V}^k_{\text{cache}}, \mathbf{V}^k_t)$

    \STATE
    \STATE \COMMENT{Step 2: Update head-averaged key block representative for sketching}
    \STATE $b_{\text{curr}} \gets \left\lceil \frac{t}{B} \right\rceil$ \COMMENT{current key-block index}
    \STATE $\mathbf{k}^k_{t,\text{avg}} \gets \frac{1}{h}\sum_{u=1}^{h}\mathbf{K}^k_t[u,:] \in \mathbb{R}^{d}$
    \IF{$t \bmod B = 1$}
        \STATE $\bar{\mathbf{k}}^k_{b_{\text{curr}}} \gets \mathbf{k}^k_{t,\text{avg}}$
        \STATE $c_{b_{\text{curr}}} \gets 1$
    \ELSE
        \STATE $\bar{\mathbf{k}}^k_{b_{\text{curr}}} \gets \frac{c_{b_{\text{curr}}}\bar{\mathbf{k}}^k_{b_{\text{curr}}}+\mathbf{k}^k_{t,\text{avg}}}{c_{b_{\text{curr}}}+1}$
        \STATE $c_{b_{\text{curr}}} \gets c_{b_{\text{curr}}}+1$
    \ENDIF

    \STATE
    \STATE \COMMENT{Step 3: Head-averaged query for sketching + Hadamard reduction}
    \STATE $\mathbf{q}^k_{t,\text{avg}} \gets \frac{1}{h}\sum_{u=1}^{h}\mathbf{Q}^k_t[u,:] \in \mathbb{R}^{d}$
    \STATE $\tilde{\mathbf{q}}^k_t \gets \mathbf{q}^k_{t,\text{avg}}\mathbf{H}_d \in \mathbb{R}^{r}$
    \STATE $\tilde{\bar{\mathbf{K}}}^k \gets [\bar{\mathbf{k}}^k_1\mathbf{H}_d;\ldots;\bar{\mathbf{k}}^k_{b_{\text{curr}}}\mathbf{H}_d] \in \mathbb{R}^{b_{\text{curr}} \times r}$
    \STATE $\hat{\mathbf{a}}^k_{\text{new}} \gets \tilde{\mathbf{q}}^k_t(\tilde{\bar{\mathbf{K}}}^k)^\top / \sqrt{r} \in \mathbb{R}^{1 \times b_{\text{curr}}}$

    \STATE
    \STATE \COMMENT{Step 4: Update cached block-attn estimate with the new row \emph{and} last column}
    \STATE $b_q \gets \left\lceil \frac{t}{B} \right\rceil$ \COMMENT{current query-block index}
    \STATE $\hat{\mathbf{A}}^{k}_{\text{block}} \gets \hat{\mathbf{A}}^{k}_{\text{block,cache}}$

    \STATE \COMMENT{(i) Update the new query row}
    \STATE $\hat{\mathbf{A}}^{k}_{\text{block}}[b_q, 1{:}b_{\text{curr}}] \gets \hat{\mathbf{a}}^k_{\text{new}}$

    \STATE \COMMENT{(ii) Update the last column for the current key block}
    \STATE $\tilde{\bar{\mathbf{Q}}}^{k} \gets [\bar{\mathbf{q}}^k_1\mathbf{H}_d;\ldots;\bar{\mathbf{q}}^k_{b_q}\mathbf{H}_d] \in \mathbb{R}^{b_q \times r}$
    \STATE $\tilde{\bar{\mathbf{k}}}^{k}_{b_{\text{curr}}} \gets \bar{\mathbf{k}}^k_{b_{\text{curr}}}\mathbf{H}_d \in \mathbb{R}^{r}$
    \STATE $\hat{\mathbf{c}}^k_{\text{new}} \gets \tilde{\bar{\mathbf{Q}}}^{k}\tilde{\bar{\mathbf{k}}}^{k}_{b_{\text{curr}}}/\sqrt{r} \in \mathbb{R}^{b_q \times 1}$
    \STATE $\hat{\mathbf{A}}^{k}_{\text{block}}[1{:}b_q, b_{\text{curr}}] \gets \hat{\mathbf{c}}^k_{\text{new}}$

    \STATE
    \STATE \COMMENT{Step 5: Random Walk}
    \STATE $\mathbf{W}^k \gets \left(\hat{\mathbf{A}}^{k}_{\text{block}}\right)^{s}$
    \IF{$k = 0$}
        \STATE $R^0 \gets \mathbf{W}^0$
    \ELSE
        \STATE $R^k \gets R^{k-1}\mathbf{W}^k$
    \ENDIF

    \STATE
    \STATE \COMMENT{Step 6: Select top-$\tau$ blocks using the new query row of the walk (always include current block)}
    \STATE $\mathcal{S} \gets \textsc{TopK-Indices}(R^k[b_q, 1{:}b_{\text{curr}}], \tau-1)\ \cup\ \{b_{\text{curr}}\}$

    \STATE
    \STATE \COMMENT{Step 7: Sparse attention on ORIGINAL per-head QKV over selected blocks}
    % \STATE Initialize $\mathbf{O}^k_t \gets \mathbf{0}^{h \times d}$
    % \FOR{$u = 1$ \textbf{to} $h$}
    %     \STATE $\mathbf{K}^{k,(u)}_{\text{sel}} \gets \textsc{GatherBlocks}(\mathbf{K}^k_{\text{cache}}[u,:,:], \mathcal{S}, B)$
    %     \STATE $\mathbf{V}^{k,(u)}_{\text{sel}} \gets \textsc{GatherBlocks}(\mathbf{V}^k_{\text{cache}}[u,:,:], \mathcal{S}, B)$
    %     \STATE $\mathbf{a}^{k,(u)} \gets \text{softmax}\!\left(\mathbf{Q}^k_t[u,:](\mathbf{K}^{k,(u)}_{\text{sel}})^\top/\sqrt{d}\right)$
    %     \STATE $\mathbf{O}^k_t[u,:] \gets \mathbf{a}^{k,(u)}\mathbf{V}^{k,(u)}_{\text{sel}}$
    % \ENDFOR

    % \STATE
    % \STATE $\hat{\mathbf{A}}^{k}_{\text{block,cache}} \gets \hat{\mathbf{A}}^{k}_{\text{block}}$ \COMMENT{persist updated block estimate}
    
    \STATE \textbf{return} $\mathbf{O}^k_t$, updated $\mathbf{K}^k_{\text{cache}}$, $\mathbf{V}^k_{\text{cache}}$, $\{\bar{\mathbf{q}}^k_i\}$, $\{\bar{\mathbf{k}}^k_j\}$, $\hat{\mathbf{A}}^{k}_{\text{block,cache}}$, $R^k$
\end{algorithmic}
\end{algorithm}

\newpage

\bibliographystyle{plainnat}
\bibliography{ref}
\end{document}